\documentclass{article}

  \usepackage[preprint]{neurips_2026}

\usepackage[utf8]{inputenc} % allow utf-8 input
\usepackage[T1]{fontenc}    % use 8-bit T1 fonts
\usepackage{hyperref}       % hyperlinks
\usepackage{url}            % simple URL typesetting
\usepackage{booktabs}       % professional-quality tables
\usepackage{amsfonts}       % blackboard math symbols
\usepackage{nicefrac}       % compact symbols for 1/2, etc.
\usepackage{microtype}      % microtypography
\usepackage{multirow}
\usepackage{tikz}

\usepackage{mathtools}
\usepackage{amssymb}
\usepackage{xspace}

\usepackage{algorithm}
\usepackage{algpseudocode}
\usepackage{pifont}
\usepackage{graphicx}

\PassOptionsToPackage{table}{xcolor}
\usepackage{xcolor}         % colors
\usepackage{colortbl}
\usepackage{wrapfig}
\usepackage{amsthm}

\usepackage{tocloft}
\usepackage{titletoc}

\newcommand{\ours}{\texttt{FedeKD}\xspace}

\definecolor{lightblue}{RGB}{220,235,255}

\newtheorem{proposition}{Proposition}

\newtheorem{corollary}{Corollary}

\title{FedeKD: Energy-Based Gating for Robust Federated Knowledge Distillation under Heterogeneous Settings}

\author{%
  Quang-Huy~Nguyen \\
  Department of Computer Science and \\
  Software Engineering \\
  Auburn University\\
  Auburn, AL 36849 \\
  \texttt{hqn0001@auburn.edu} \\
  \And
   Jiaqi~Wang\thanks{Co-corresponding authors.} \\
  Department of Computer Science and \\
  Software Engineering \\
  Auburn University\\
  Auburn, AL 36849 \\
  \texttt{jqwang@auburn.edu} \\
  \And
   Wei-Shinn~Ku$^{\ast}$ \\
  Department of Computer Science and \\
  Software Engineering \\
  Auburn University\\
  Auburn, AL 36849 \\
  \texttt{wzk0004@auburn.edu} \\
}

\begin{document}

\maketitle

\begin{abstract}
Federated learning (FL) operates in heterogeneous environments, where variations in data distributions and asymmetric model design often result in negative transfer.
While federated knowledge distillation (FKD) avoids direct model parameter sharing, existing methods typically rely on public datasets or assume that transferred knowledge is uniformly reliable, which limits their robustness in practice.
This paper presents \ours, a reliability-aware FKD framework that makes sample-wise trust estimation an explicit component of knowledge transfer, without relying on additional public data. Each client maintains a high-capacity private model for local learning and a lightweight shared proxy model for cross-client knowledge exchange. During training, proxy models are aggregated on the server to form a global proxy, which is then used to guide updates of the private models.
At the core of \ours is an energy-based gating mechanism that converts task-specific private-proxy disagreement into sample-wise trust weights for backward distillation. %, using entropy-calibrated distributional disagreement for classification and continuous prediction disagreement for regression.
This mechanism enables sample-wise weighting of knowledge transfer, where the proxy model contributes more to reliable samples while down-weighting unreliable ones.
Extensive experiments on six real-world datasets demonstrate that \ours significantly reduces negative transfer under heterogeneous settings while maintaining strong predictive performance.
\end{abstract}

\section{Introduction}

Federated learning (FL) enables multiple clients to collaboratively train models without sharing raw data, making it well-suited for privacy-sensitive applications. However, real-world deployments involve significant heterogeneity in data distributions and model asymmetry, making reliable knowledge sharing challenging. Canonical parameter aggregation methods, including FedAvg~\cite{mcmahan2017communication} and FedProx~\cite{li2020federated}, often suffer from degraded performance and unstable convergence under such heterogeneity.

Federated knowledge distillation (FKD) offers an alternative by transferring knowledge through model outputs rather than parameters. However, in heterogeneous settings, the quality of transferred knowledge can vary significantly across samples and clients. This calls for treating knowledge reliability as a first-class object in FKD, rather than assuming that all transferred signals should influence local learning equally. Teacher models may produce misleading signals due to distribution shifts, limited local data, or architectural differences. Blindly transferring such knowledge can lead to severe negative transfer (i.e., performance degradation compared to local training) across clients. 
Existing FKD methods typically assume that transferred knowledge is uniformly reliable, overlooking the variability in knowledge quality across samples and clients.
This challenge mirrors human learning, where guidance becomes unreliable when a teacher's expertise is misaligned with the learner's context or when the guidance provides little actionable information.

%Wang et al.~\cite{wang2024bridging} propose an asymmetric private–proxy FL paradigm, where each client maintains a high-capacity private model and a lightweight shared proxy model for communication. This design removes the need for public data, reduces communication cost through compact proxies, and avoids sharing private model parameters or sensitive intermediate representations. However, it does not explicitly model the reliability of transferred knowledge at the sample level, limiting its ability to prevent harmful transfer in heterogeneous settings (see Section~\ref{sec:related_work}).
In this work, we seek to answer the fundamental question: \textit{when should a teacher model be trusted to guide knowledge transfer?} Our key insight is inspired by a simple real-world principle: a teacher does not treat all knowledge equally, but instead emphasizes the parts they understand best and are most confident in.
Building on this insight, we introduce a \underline{\textbf{Fed}}erated learning framework with \underline{\textbf{e}}nergy-gated \underline{\textbf{K}}nowledge \underline{\textbf{D}}istillation (\ours) to dynamically determine how much each proxy prediction should influence each private-model update (Figure~\ref{fig:figure1}). \ours achieves robust knowledge transfer across heterogeneous clients and mitigates negative transfer through adaptive, sample-wise weighting.

\begin{figure*}[t]
\centering
    \includegraphics[width=\textwidth]{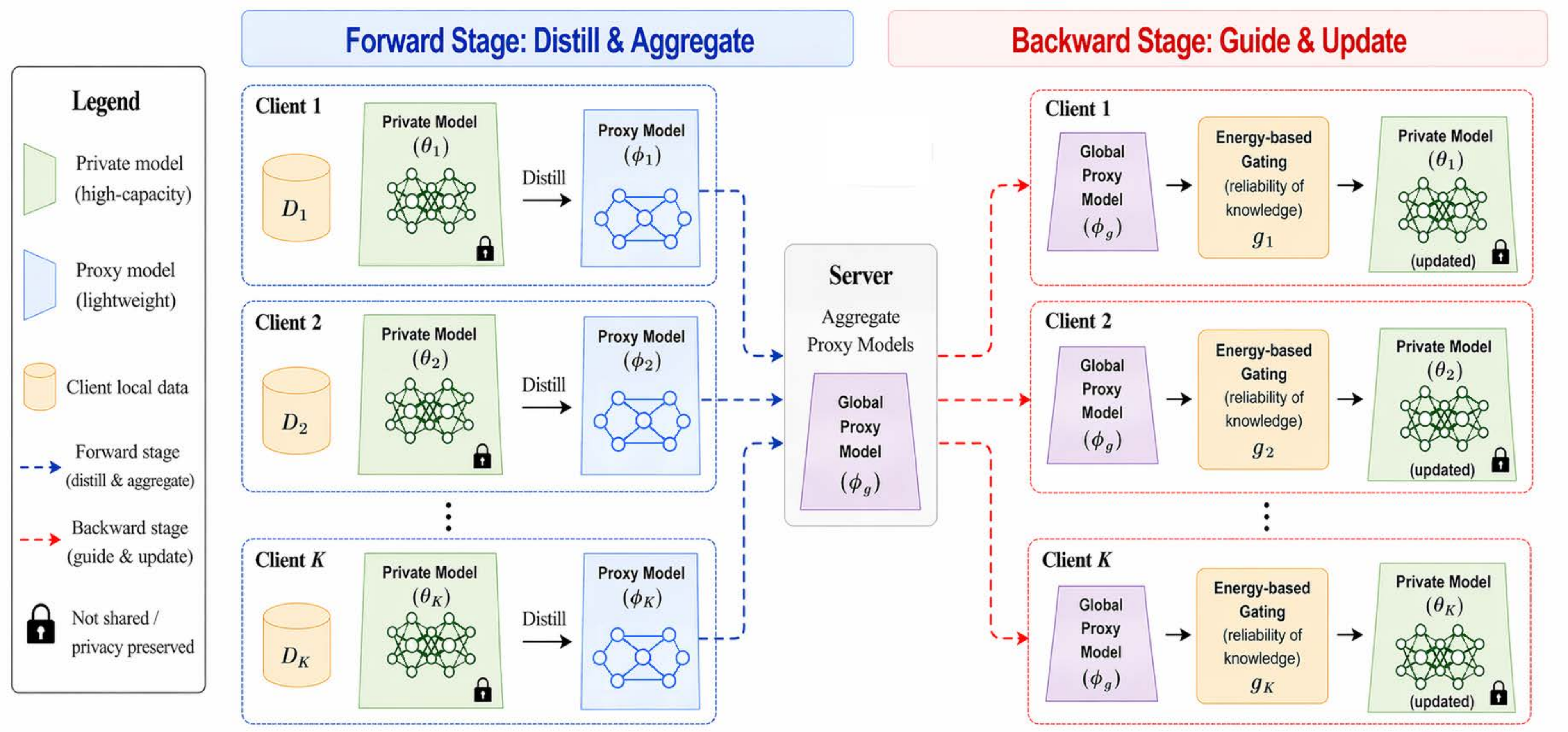}
    \caption{Framework of \ours. Private Model denotes a higher-capacity network for local learning, while Proxy Model is a lightweight network used for communication and aggregation across clients. See Appendix~\ref{app:study_design} for details.}
    \label{fig:figure1}
\end{figure*}

\ours operates in two stages. In the \underline{forward} stage, each client distills knowledge from its private model into a proxy model, which is aggregated on the server to form a global proxy.
In the \underline{backward} stage, the global proxy guides updates of the private model through an energy-gated mechanism. This process enables \textbf{sample-wise trust-weighted knowledge transfer}, where the model relies more on the proxy when the transferred knowledge appears reliable and less when it appears unreliable.

To sum up, the contributions of this work are twofold:
\begin{itemize}
    \item We introduce a reliability-aware backward distillation objective that shifts the focus from how to exchange knowledge to how much each transferred signal should be trusted. To support this objective without requiring public data, we design \ours with an asymmetric private-proxy architecture. The objective is implemented via a batch-normalized energy gate that maps task-specific private-proxy disagreement into continuous sample-wise trust weights. For classification, the energy uses entropy-calibrated distributional disagreement; for regression, it uses continuous prediction disagreement. This mechanism down-weights unreliable knowledge while preserving informative signals, thereby mitigating negative transfer in heterogeneous FL environments.
    
    \item Extensive experiments on six real-world datasets demonstrate that \ours significantly improves both average-case and worst-case negative transfer while maintaining strong predictive performance. Additional ablation studies show that the proposed gating mechanism remains effective across different heterogeneity levels and hyperparameter settings.
\end{itemize}
\section{Related Works} \label{sec:related_work}
\paragraph{Parameter aggregation methods.}
Parameter aggregation remains the dominant paradigm in FL, where a global model is obtained by aggregating locally trained client models. FedAvg~\cite{mcmahan2017communication} and FedProx~\cite{li2020federated} have demonstrated strong empirical performance under homogeneous settings. FedDyn~\cite{durmus2021federated} further improves robustness under heterogeneous data by introducing a dynamic regularization term that aligns local and global objectives during training.
These methods assume that all clients share a common model architecture and operate under relatively homogeneous data distributions, allowing model updates to be directly aggregated to improve global performance. However, such assumptions are often violated in real-world FL environments. In heterogeneous settings, direct parameter aggregation often fails due to distribution mismatch across clients, leading to degraded convergence and suboptimal performance. In addition, sharing model parameters may expose sensitive information about local data distributions, raising potential privacy concerns in FL settings.

In contrast, \ours avoids direct aggregation of private model parameters and instead performs aggregation in the proxy space, using knowledge distillation as the primary mechanism for knowledge exchange. This design reduces the exposure of private model parameters, thereby enhancing privacy preservation in federated settings. Furthermore, instead of assuming uniform reliability across clients or samples, \ours approximates the quality of transferred knowledge via task-specific model disagreement, using entropy calibration for classification and continuous prediction disagreement for regression.

\paragraph{Federated Knowledge Distillation and Heterogeneous FL.}
To address the limitations of parameter aggregation under heterogeneous settings, recent works have explored alternative communication paradigms based on knowledge distillation. Instead of directly aggregating model parameters, these methods exchange auxiliary information such as logits~\cite{huang2022learn}, class scores~\cite{li2019fedmd}, or label-wise representations~\cite{yi2023fedgh, tan2022fedproto} to facilitate collaboration across heterogeneous models. More recent approaches further extend this idea through techniques such as  ensemble learning~\cite{lin2020ensemble}, mutual learning~\cite{yu2022resource, shen2023federated}, or model reassembly~\cite{wang2023towards}. While these methods improve flexibility in heterogeneous FL, they often rely on additional public datasets or shared data representations to stabilize training. However, such reliance introduces two key limitations in practice. (1) The availability and selection of suitable public data remain challenging in real-world applications. (2) Exchanging intermediate representations or model-related information may expose sensitive information about local data distributions, raising privacy concerns. Moreover, these approaches typically assume that transferred knowledge is uniformly reliable, overlooking the variability in knowledge quality across samples and clients. In contrast, \ours does not require additional public data and restricts communication to lightweight proxy models.

FedType~\cite{wang2024bridging} is the closest prior work, as it also considers asymmetrical reciprocity between small proxy models and large client models. However, \ours differs from FedType in several key aspects.
\underline{First}, the key methodological difference lies in how reliability is represented, optimized, and coupled to the private-model update. FedType relies on two conformal models to construct discrete, sample-dependent uncertainty sets for the client and proxy models, whereas %, which do not explicitly capture sample-wise variations across clients. 
\ours replaces set-valued reliability filtering with a continuous training objective: task-specific private-proxy disagreement directly modulates the backward distillation loss through sample-level trust weights (Eq.~\ref{eq:trust_weight}). This formulation provides finer-grained control over knowledge transfer at the sample level, allowing the model to smoothly down-weight unreliable signals rather than relying on set-based filtering.
\underline{Second}, FedType's specific conformal-set reciprocity formulation is classification-oriented, as it is defined over discrete label spaces. In contrast, \ours defines reliability directly on model outputs, enabling the same gating principle to apply to classification via distributional disagreement and to regression via continuous prediction disagreement.
\section{FedeKD}
%Each client performs a supervised update on its private model.
%The proxy model is trained via forward knowledge distillation
%to mimic the private model outputs on the local data.

\subsection{Problem Setup}

We consider a FL setting with $K$ clients. Each client $k$ has access to a local dataset $\mathcal{D}_k = \{(x_i, y_i)\}$, which is not shared with other clients or the server. The goal is to collaboratively improve local models while preserving data privacy under heterogeneous data distributions. Each client maintains two models: a private model $f_k$ for the primary learning task and a lightweight proxy model $g_k$ for cross-client knowledge exchange. The proxy models share a common architecture across clients and are aggregated on the server, while private models remain local and are never shared.

\subsection{Overview of \ours}

At each communication round, \ours implements three sequential stages: forward proxy distillation, proxy aggregation, and energy-gated private-model update. In the forward stage, each client keeps its private model fixed and trains only the lightweight proxy model to mimic the current private model on local data. The proxy models are then uploaded to the server and aggregated to form a global proxy model, which is broadcast back to all clients. In the backward stage, each client updates its private model once using a combined objective consisting of the supervised loss and the energy-gated backward distillation loss from the global proxy. Algorithm~\ref{alg:fedekd} summarizes one communication round.

\subsection{Forward Proxy Distillation}

%Each client begins by updating its private model $f_k$ using supervised learning. The objective is defined as $\mathcal{L}_{\text{sup}} = \mathbb{E}_{(x,y)\sim \mathcal{D}_k} [\ell(f_k(x), y)]$, where $\ell(\cdot)$ denotes the task-specific loss. For classification tasks, $\ell$ corresponds to the cross-entropy loss, while for regression tasks, it is instantiated as mean squared error. 
%The proxy model $g_k$ is then trained to mimic the private model via forward knowledge distillation. For classification, this is achieved by minimizing $\mathcal{L}_{\text{FKD}} = \mathbb{E}_{x \sim \mathcal{D}_k} \big[ \mathrm{KL}(\sigma(f_k(x)) \,\|\, \sigma(g_k(x))) \big]$, where $\sigma(\cdot)$ denotes the softmax function. For regression, the distillation objective reduces to a squared error loss $\mathcal{L}_{\text{FKD}} = \mathbb{E}_{x \sim \mathcal{D}_k} \big[ \| f_k(x) - g_k(x) \|^2 \big]$. This step enables the proxy model to capture task-relevant knowledge from the private model while maintaining a shared representation space across clients.

At the beginning of each communication round, each client keeps its private model $f_k$ fixed and trains only the proxy model $g_k$ on local data. The proxy is trained to mimic the private model via forward knowledge distillation using a task-specific distillation loss. For classification, this corresponds to matching predictive distributions, while for regression, it reduces to matching continuous predictions. This step enables the proxy model to capture task-relevant knowledge from the private model while maintaining a shared representation space across clients.

\subsection{Proxy Aggregation}

After local updates, each client uploads its proxy model to the server. The server aggregates these proxy models to form a global proxy, which is then broadcast back to all clients. Importantly, aggregation is performed only in the proxy space, avoiding direct sharing of private model parameters and thereby reducing the exposure of sensitive information.

\subsection{Energy-Gated Backward Distillation}

To transfer knowledge from the global proxy back to the private model, \ours introduces a reliability-aware backward distillation objective that operationalizes sample-wise trust through energy-based gating. Given an input $x$, an energy score $E(x)$ is computed to measure the disagreement between the private and proxy models, i.e., $E(x) = \mathcal{E}(g(x), f_k(x))$, where $\mathcal{E}(\cdot)$ can be instantiated using various discrepancy measures.

For classification, we adopt an entropy-normalized symmetric KL divergence defined as
\begin{equation}
E(x)
=
\frac{\frac{1}{2}\big( \mathrm{KL}(p \,\|\, q) + \mathrm{KL}(q \,\|\, p) \big)}{H(p)+H(q)+\epsilon_H},
\label{eq:classification_energy}
\end{equation}
where $p=\sigma(f_k(x))$, $q=\sigma(g(x))$, $H(\cdot)$ denotes entropy, and $\epsilon_H=10^{-8}$ is a small numerical constant. 
%This formulation captures both prediction disagreement and uncertainty. 
This formulation measures private-proxy disagreement while calibrating its magnitude by the total predictive entropy of the two models.
For regression, we use squared error as a proxy for prediction disagreement, and the energy reduces to
\begin{equation}
E(x)=\frac{1}{2}\|f_k(x)-g(x)\|_2^2,
\label{eq:regression_energy}
\end{equation}
which provides a continuous measure of point-prediction disagreement. For the scalar-output regression tasks used in our experiments, this reduces to the squared difference between the private and proxy predictions. Unlike the classification energy, this regression energy does not model predictive uncertainty; instead, it measures functional disagreement between the private and proxy predictions.

\paragraph{Design rationale.}
The proposed classification energy function is based on a symmetric KL divergence normalized by the total entropy, as defined in Eq.~\eqref{eq:classification_energy}.
%This formulation has three key properties. First, the symmetric KL divergence ensures that the disagreement measure is invariant to the ordering of the private and proxy models, i.e., $E(p,q)=E(q,p)$. Second, normalization by the sum of entropies calibrates disagreement by the total predictive uncertainty. In particular, it reduces the tendency of high-entropy predictions to dominate the energy magnitude, thereby ensuring that disagreement is interpreted relative to the confidence of the models. Finally, this normalized form jointly captures both prediction mismatch and uncertainty, enabling more robust handling of low-quality knowledge.
This formulation has three key properties. First, the symmetric KL divergence ensures that the disagreement measure is invariant to the ordering of the private and proxy models, i.e., $E(p,q)=E(q,p)$. Second, entropy normalization makes the score a relative measure of disagreement rather than an absolute KL magnitude. For a fixed symmetric KL value, the normalized energy is larger when the private and proxy predictions are sharp and low-entropy, and smaller when their predictions are diffuse and high-entropy. Thus, the gate emphasizes confident contradictions between the private and proxy models, which are more likely to induce harmful transfer, while treating diffuse high-entropy disagreements as less decisive. Finally, the normalized form combines distributional mismatch with prediction specificity, enabling reliability-aware weighting of transferred knowledge.

The energy scores are then converted into sample-wise trust weights within each minibatch. For a minibatch $B=\{x_i\}_{i=1}^{|B|}$, let $E_i=E(x_i)$, $\mu_B=|B|^{-1}\sum_{i=1}^{|B|}E_i$, and $s_B=(|B|^{-1}\sum_{i=1}^{|B|}(E_i-\mu_B)^2)^{1/2}$. We first form a batch-normalized energy and then apply a logistic gate:
\begin{equation}
\widetilde{E}_i
=
\frac{E_i-\mu_B}{s_B+\epsilon_B},
\qquad
w_i
=
\rho(-\beta\widetilde{E}_i),
\qquad
\rho(t)=\frac{1}{1+\exp(-t)},
\label{eq:trust_weight}
\end{equation}
where $\epsilon_B=10^{-8}$ is a small numerical constant and $\beta>0$ controls the sharpness of the gate. This gate makes reliability batch-relative: a sample is trusted according to
how safe its transfer signal appears compared with other samples in the same minibatch, rather than according to a fixed global threshold.

The private model is updated using a weighted distillation objective defined as
\begin{equation}
\mathcal{L}_{\text{BKD}}
=
\mathbb{E}_{B\sim \mathcal{D}_k}
\left[
\frac{1}{|B|}
\sum_{i=1}^{|B|}
\operatorname{sg}(w_i)\ell_{\text{KD}}(x_i)
\right],
\label{eq:bkd_loss}
\end{equation}
where $\ell_{\text{KD}}(x_i)$ denotes the distillation loss between the proxy and private models, and $\operatorname{sg}(\cdot)$ denotes stop-gradient. We use $\ell_{\text{KD}}(x_i)=\mathrm{KL}(q_i\|p_i)$ for classification, with $p_i=\sigma(f_k(x_i))$ and $q_i=\sigma(g(x_i))$, and squared error for regression. The full private-model update combines the supervised loss with the gated distillation term:
\begin{equation}
\mathcal{L}_{\text{private}}
=
\mathcal{L}_{\text{sup}}
+
\lambda_{\text{kd}}\mathcal{L}_{\text{BKD}}.
\label{eq:private_loss}
\end{equation}
The stop-gradient operator ensures that the gate acts as a reliability weight rather than introducing an additional optimization path through the energy function. Thus, for each sample, the gate rescales the distillation gradient but does not reverse the underlying distillation direction. We provide stability-oriented properties of this mechanism in Appendix~\ref{sec:gate_theory}, including batch-relative monotonicity, bounded and non-degenerate trust weights, a variational interpretation of the logistic gate, and output-level influence results for both classification and regression.

Intuitively, samples with low relative disagreement are assigned higher trust weights, while high relative-disagreement samples are down-weighted. In practice, this mechanism modulates the influence of the proxy model on each sample by reweighting the distillation loss according to the relative level of agreement between the private and proxy models.

\paragraph{Properties of the energy-based gating function.}

The batch-normalized logistic gate should be interpreted as a relative trust assignment within each minibatch. Samples with energy below the batch mean receive weights above $1/2$, whereas samples with energy above the batch mean receive weights below $1/2$. Moreover, for any two samples in the same minibatch, $E_i\le E_j$ implies $w_i\ge w_j$, so the gate preserves the intended ordering of reliability. This batch-relative monotonicity property is formalized in Proposition~\ref{prop:batch_relative_monotonicity}.
The logistic gate also produces bounded and non-degenerate trust weights. In particular, the weights remain strictly between $0$ and $1$, avoiding hard rejection while still reducing the influence of high-energy samples. This property is formalized in Proposition~\ref{prop:nondegenerate_trust_weights}. In addition, the logistic form admits a variational interpretation as the solution to an entropy-regularized soft trust assignment problem, as shown in Proposition~\ref{prop:logistic_gate_variational}.
Finally, because the gate is detached during backpropagation, it scales the distillation loss without introducing an additional gradient path through the energy function. As a result, the gated per-sample distillation gradient is a positive scalar multiple of the ungated distillation gradient and cannot reverse its direction. We formalize this behavior in Proposition~\ref{prop:detached_gate_direction}, with output-level influence results for regression and classification in Corollaries~\ref{cor:regression_output_influence} and~\ref{cor:classification_output_influence}. Together, these stability properties connect the proposed objective to the empirical robustness of \ours: unreliable proxy guidance is not discarded by a hard rule, but is systematically assigned lower influence while useful distillation signals remain active.

%\paragraph{Observation.} The weighted distillation loss satisfies $\mathcal{L}_{\text{BKD}} = \mathbb{E}[w(x)\ell(x)] \le \mathbb{E}[\ell(x)]$.

These properties help explain the empirical robustness of \ours observed in Section~\ref{sec:benchmark}, where unreliable samples tend to receive lower relative weights across heterogeneous settings.

\section{Benchmark} \label{sec:benchmark}
\subsection{Settings}
\paragraph{Benchmarks and Data Partitioning.} 
We evaluate the effectiveness and scalability of \ours on six real-world datasets most frequently used by prior FL works~\cite{lu2023federated,wang2024bridging,minpersonalized,nguyen2026conformalizedneuralnetworksfederated}, ensuring a fair comparison with established baselines. The six datasets span two tasks: classification (FashionMNIST~\cite{xiao2017fashion}, CIFAR-10~\cite{krizhevsky2009learning}, OCTMNIST~\cite{kermany2018identifying}, and OrganAMNIST~\cite{xu2019efficient}) and regression (RetinaMNIST~\cite{dataset20202nd} and Diabetic Retinopathy (Regular Fundus)~\cite{woerner2024comprehensive}). To simulate realistic data heterogeneity, we apply a Dirichlet partition with concentration parameter $\alpha$ to induce label shift in classification and covariate shift in regression tasks. Smaller values of $\alpha$ correspond to more heterogeneous data distributions across clients. Detailed data statistics and partition descriptions are provided in Appendix~\ref{app:dat_stats}.

\paragraph{Experimental Design for Heterogeneity.}
Following prior work in cross-silo FL~\cite{wang2025asymmetrical,nguyen2026conformalizedneuralnetworksfederated}, we consider a system with six clients and adopt an asymmetric model design within each client, consisting of a high-capacity private model and a lightweight proxy model. All clients share identical model architectures. Model asymmetry arises from the distinct roles and capacities of the private and proxy models, which create discrepancies between models during knowledge transfer. Further details on the model architectures and the controlled heterogeneity setup are provided in Appendix~\ref{app:study_design}.

\paragraph{Hyperparameter settings.}
For \ours, we use $\beta = 1$ and $\lambda_{\text{kd}} = 1$ across all experiments without additional tuning. All baselines are evaluated using their standard or publicly reported hyperparameter configurations. Models are trained using Adam with a learning rate of $10^{-4}$, using $5$ communication rounds with $2$ local epochs per round across all methods to ensure consistent and computationally comparable evaluation. The local minibatch size is set to $B=64$ for all training procedures, while a larger batch size of 256 is used for evaluation.

\subsection{Evaluation Metrics}

We evaluate model performance under both classification and regression settings with a focus on negative transfer and robustness across heterogeneous clients, while also reporting accuracy and RMSE as measures of predictive performance.

\paragraph{Classification.}
We report accuracy, defined as $\text{Acc} = \frac{1}{N} \sum_{i=1}^{N} \mathbb{I}(\hat{y}_i = y_i)$, where $\hat{y}_i$ and $y_i$ denote the predicted and ground-truth labels. To explicitly quantify negative transfer, we compute the performance change relative to local training as $\Delta_k = \text{Acc}_k^{\text{method}} - \text{Acc}_k^{\text{local}}$, where negative values indicate degradation. We aggregate these values across clients using their average (Avg $\Delta$) and worst case (Worst $\Delta$), and further report the 10th percentile ($\text{P10 } \Delta$) over clients to capture lower-tail behavior. 
%We also report RHR, defined as $\text{RHR} = \frac{\mathbb{E}[|\Delta^{\text{naive}}|] - \mathbb{E}[|\Delta^{\text{method}}|]}{\mathbb{E}[|\Delta^{\text{naive}}|]}$. 
Client-level robustness is additionally assessed using the average $\text{Avg} = \frac{1}{K} \sum_{k=1}^{K} \text{Acc}_k$ and worst-case $\text{Worst} = \min_k \text{Acc}_k$ accuracy.

\paragraph{Regression.}
We report root mean squared error (RMSE), defined as $\text{RMSE}=\sqrt{\frac{1}{N}\sum_{i=1}^{N}(\hat{y}_i-y_i)^2}$. Following the same evaluation principle, we define $\Delta_k = \text{RMSE}_k^{\text{method}} - \text{RMSE}_k^{\text{local}}$, where positive values indicate degradation. In this setting, we focus on the upper tail of the distribution and report the 90th percentile ($\text{P90 } \Delta$) over clients, along with the average (Avg. $\Delta$) and worst-case (Worst $\Delta$) values of $\Delta_k$. We measure robustness across clients using the average $\text{Avg} = \frac{1}{K} \sum_{k=1}^{K} \text{RMSE}_k$ and the worst-case $\text{Worst} = \max_k \text{RMSE}_k$.

\subsection{Robustness to Negative Transfer}

\begin{table}[!t]
\caption{
Measurement of average-case negative transfer (Avg $\Delta$) for classification tasks (higher is better). 
Results are reported as mean$_{\pm\text{std}}$ over 10 independent runs.
\textbf{Bold} and \underline{underline} denote the best and second-best results in each column, respectively. 
\ours consistently achieves the strongest reduction in negative transfer under severe heterogeneity ($\alpha = 0.1$), remains highly competitive at moderate heterogeneity ($\alpha = 0.3$), and performs comparably to the best alternatives when data distributions become more homogeneous ($\alpha = 0.5$). Notably, \ours exhibits consistently lower variance across runs, indicating more stable and reliable performance compared to the baselines.
}
\label{tab:neg_transfer_main}

\centering
\small
\setlength{\tabcolsep}{3pt}
\resizebox{\textwidth}{!}{
\begin{tabular}{c|c|ccc|ccc|ccc|ccc}
\hline
\textbf{\shortstack{Agg.\\Method}} & \textbf{Model}
& \multicolumn{3}{c|}{\textbf{FashionMNIST}}
& \multicolumn{3}{c|}{\textbf{CIFAR-10}}
& \multicolumn{3}{c|}{\textbf{OCTMNIST}}
& \multicolumn{3}{c}{\textbf{OrganAMNIST}} \\

\cline{3-14}
& & $\alpha=0.1$ & $\alpha=0.3$ & $\alpha=0.5$
  & $\alpha=0.1$ & $\alpha=0.3$ & $\alpha=0.5$
  & $\alpha=0.1$ & $\alpha=0.3$ & $\alpha=0.5$
  & $\alpha=0.1$ & $\alpha=0.3$ & $\alpha=0.5$ \\

\hline

FedDyn & -
& $-0.8309_{\pm 0.0500}$ & $-0.7933_{\pm 0.0187}$ & $-0.7666_{\pm 0.0143}$
& $-0.6560_{\pm 0.0642}$ & $-0.5738_{\pm 0.0390}$ & $-0.5285_{\pm 0.0302}$
& $-0.6328_{\pm 0.1168}$ & $-0.5242_{\pm 0.0753}$ & $-0.4790_{\pm 0.0527}$
& $-0.8317_{\pm 0.0623}$ & $-0.8252_{\pm 0.0159}$ & $-0.8080_{\pm 0.0247}$ \\

FedProx & -
& $-0.4506_{\pm 0.0700}$ & $-0.2468_{\pm 0.0645}$ & $-0.1821_{\pm 0.0390}$
& $-0.5553_{\pm 0.0775}$ & $-0.3480_{\pm 0.0403}$ & $-0.2788_{\pm 0.0328}$
& $-0.4989_{\pm 0.1108}$ & $-0.2965_{\pm 0.1259}$ & $-0.1909_{\pm 0.0424}$
& $-0.5313_{\pm 0.1100}$ & $-0.3061_{\pm 0.0629}$ & $-0.2287_{\pm 0.0364}$ \\

FedProx & FedType
& $-0.0096_{\pm 0.0092}$ & $-0.0129_{\pm 0.0076}$ & $-0.0145_{\pm 0.0058}$
& $-0.0054_{\pm 0.0174}$ & $-0.0042_{\pm 0.0067}$ & \underline{$0.0011_{\pm 0.0049}$}
& $-0.0264_{\pm 0.0168}$ & $-0.0297_{\pm 0.0164}$ & $-0.0320_{\pm 0.0221}$
& $-0.0122_{\pm 0.0110}$ & $-0.0030_{\pm 0.0114}$ & \underline{$-0.0016_{\pm 0.0090}$} \\

FedAvg & -
& $-0.4604_{\pm 0.0843}$ & $-0.2545_{\pm 0.0629}$ & $-0.1921_{\pm 0.0353}$
& $-0.5525_{\pm 0.0793}$ & $-0.3540_{\pm 0.0431}$ & $-0.2840_{\pm 0.0308}$
& $-0.5044_{\pm 0.1116}$ & $-0.3003_{\pm 0.1252}$ & $-0.2001_{\pm 0.0448}$
& $-0.5482_{\pm 0.1082}$ & $-0.3193_{\pm 0.0669}$ & $-0.2408_{\pm 0.0404}$ \\

FedAvg & FedType
& \underline{$-0.0087_{\pm 0.0089}$} & \underline{$-0.0125_{\pm 0.0083}$} & \underline{$-0.0130_{\pm 0.0056}$}
& \underline{$-0.0049_{\pm 0.0120}$} & \underline{$-0.0023_{\pm 0.0062}$} & $\mathbf{0.0030}_{\pm 0.0051}$
& \underline{$-0.0247_{\pm 0.0105}$} & \underline{$-0.0284_{\pm 0.0178}$} & \underline{$-0.0306_{\pm 0.0183}$}
& \underline{$-0.0116_{\pm 0.0108}$} & $\mathbf{-0.0014}_{\pm 0.0120}$ & $\mathbf{0.0010}_{\pm 0.0079}$ \\ 

\rowcolor{gray!12}
FedAvg & \textbf{\ours}
& $\mathbf{-0.0065}_{\pm 0.0066}$ & $\mathbf{-0.0035}_{\pm 0.0050}$ & $\mathbf{0.0031}_{\pm 0.0035}$
& $\mathbf{-0.0033}_{\pm 0.0098}$ & $\mathbf{-0.0017}_{\pm 0.0051}$ & $-0.0016_{\pm 0.0048}$
& $\mathbf{-0.0040}_{\pm 0.0084}$ & $\mathbf{-0.0013}_{\pm 0.0089}$ & $\mathbf{-0.0024}_{\pm 0.0096}$
& $\mathbf{-0.0078}_{\pm 0.0060}$ & \underline{$-0.0021_{\pm 0.0064}$} & $-0.0024_{\pm 0.0048}$ \\

\hline
\end{tabular}
}
\end{table}

\begin{table}[!t]
\caption{
Measurement of Avg $\Delta$ for regression tasks (lower is better). 
\ours consistently and stably achieves the lowest Avg $\Delta$.
}
\label{tab:neg_transfer_regression}

\centering
\small
\setlength{\tabcolsep}{3pt}
\resizebox{\textwidth}{!}{
\begin{tabular}{c|c|ccc|ccc}
\hline
\textbf{\shortstack{Agg.\\Method}} & \textbf{Model}
& \multicolumn{3}{c|}{\textbf{RetinaMNIST}}
& \multicolumn{3}{c}{\textbf{Diabetic Retinopathy}} \\

\cline{3-8}
& & $\alpha=0.1$ & $\alpha=0.3$ & $\alpha=0.5$
  & $\alpha=0.1$ & $\alpha=0.3$ & $\alpha=0.5$ \\

\hline

FedDyn & -
& $0.6091_{\pm 0.2355}$ & $0.8470_{\pm 0.2508}$ & $0.7308_{\pm 0.1350}$
& $1.3174_{\pm 2.3447}$ & $0.7513_{\pm 0.3794}$ & $0.6559_{\pm 0.1899}$ \\

FedProx & -
& \underline{$0.1445_{\pm 0.0714}$} & $0.1509_{\pm 0.0941}$ & $0.1145_{\pm 0.0409}$
& $0.1248_{\pm 0.0765}$ & $0.1700_{\pm 0.0764}$ & \underline{$0.1506_{\pm 0.0590}$} \\

FedAvg & -
& $0.1459_{\pm 0.0726}$ & \underline{$0.1467_{\pm 0.0942}$} & \underline{$0.1131_{\pm 0.0403}$}
& \underline{$0.1216_{\pm 0.0739}$} & \underline{$0.1682_{\pm 0.0759}$} & $0.1507_{\pm 0.0585}$ \\

\rowcolor{gray!12}
FedAvg & \textbf{\ours}
& $\mathbf{-0.0019}_{\pm 0.0474}$ & $\mathbf{0.0034}_{\pm 0.0619}$ & $\mathbf{-0.0178}_{\pm 0.0353}$
& $\mathbf{-0.0606}_{\pm 0.0935}$ & $\mathbf{-0.0050}_{\pm 0.0611}$ & $\mathbf{-0.0058}_{\pm 0.0798}$ \\

\hline
\end{tabular}
}
\end{table}

Tables~\ref{tab:neg_transfer_main} and~\ref{tab:neg_transfer_regression} report the primary average-case negative-transfer results for classification and regression, respectively. To assess whether these gains also hold beyond the mean, we further examine worst-case and tail-sensitive metrics reported in Tables~\ref{tab:neg_transfer_worst}, \ref{tab:neg_transfer_worst_reg}, \ref{tab:neg_transfer_p10}, and \ref{tab:p90_delta_reg} in Appendices.

\paragraph{Classification.}
For classification, \ours delivers the strongest overall robustness profile against negative transfer across four datasets and multiple heterogeneity levels ($\alpha$). On the main metric Avg $\Delta$ in Table~\ref{tab:neg_transfer_main}, \ours is best in all four datasets at $\alpha=0.1$, remains best on three datasets at $\alpha=0.3$, and stays competitive at $\alpha=0.5$, where the gap between methods naturally shrinks as client distributions become less skewed. This pattern indicates that reliability-aware distillation is most beneficial when transferred knowledge is least trustworthy, particularly under strong data heterogeneity.

A stronger robustness signal appears when evaluation moves beyond averages to the vulnerable-client regime. In Table~\ref{tab:neg_transfer_worst}, \ours consistently achieves the best Worst $\Delta$ across all datasets and all heterogeneity levels, showing that it most effectively limits the most severe degradation experienced by any client. Likewise, Table~\ref{tab:neg_transfer_p10} shows that \ours also dominates P10 $\Delta$, which captures the lower tail of the client distribution rather than only the single worst case. Together, these two results indicate that the gains of \ours are not driven by a few easy clients or by averaging effects. Instead, the method improves robustness throughout the lower end of the client population.
%In Table~\ref{tab:rhr_main}, \ours attains the highest Avg RHR across all datasets and all heterogeneity levels. More notably, Table~\ref{tab:worst_rhr_main} shows that \ours also achieves the highest Worst RHR throughout, indicating the most favorable relative harm profile among the compared methods, even for the most negatively affected clients. These appendix results are important because they normalize improvement relative to the severity of the original degradation. 
%They therefore confirm that \ours is not only better in absolute terms, but is also the most effective at suppressing negative transfer relative to the baseline damage.
These results show that \ours not only improves average-case performance but also provides consistent robustness gains across the lower tail of the client population.

%A more nuanced comparison emerges against FedType-based variants. Under mild heterogeneity, FedAvg+FedType and FedProx+FedType can become competitive on some average-case entries in Table~\ref{tab:neg_transfer_main}, and in a few cells they slightly edge out \ours. However, this advantage does not survive the more stringent robustness metrics. On Worst $\Delta$, P10 $\Delta$, Avg RHR, and Worst RHR, \ours remains consistently stronger. This distinction is important: in FL, a method that is competitive on the mean but unstable in the tail is often much less desirable than one that offers uniform protection across clients. The evidence across Tables~\ref{tab:neg_transfer_worst}, \ref{tab:neg_transfer_p10}, \ref{tab:rhr_main}, and \ref{tab:worst_rhr_main} shows that \ours achieves precisely this stronger form of robustness.

\paragraph{Regression.}
On the primary metric Avg $\Delta$ in Table~\ref{tab:neg_transfer_regression}, \ours is best in every setting across both RetinaMNIST and Diabetic Retinopathy and across all heterogeneity levels. Standard FL baselines remain consistently harmful, while \ours is the only method that reliably drives average degradation toward zero or below. 

This advantage becomes more pronounced under stronger robustness metrics. Table~\ref{tab:neg_transfer_worst_reg} shows that \ours achieves the lowest Worst $\Delta$ in every setting, substantially reducing the severity of the most extreme degradation. Similarly, Table~\ref{tab:p90_delta_reg} shows the same trend for P90 $\Delta$, confirming that the benefit extends to the upper tail of the error distribution and is not confined to a single pathological client. In other words, \ours improves both the center and the tail of the regression negative-transfer distribution.

%These results are particularly significant in regression, where harmful transfer can propagate continuously and accumulate without the implicit regularization induced by discrete class structure. The fact that \ours is uniformly best on Avg $\Delta$, Worst $\Delta$, and P90 $\Delta$ suggests that its sample-wise trust mechanism is especially well matched to continuous-output settings, where it can effectively attenuate misleading proxy signals before they distort the private model.

%Overall, the evidence across the main text and appendix supports a consistent conclusion. The key benefit of \ours is not simply a marginal gain in average performance, but a broad improvement in robustness to harmful knowledge transfer. This is most visible under strong heterogeneity, where naively shared information is least reliable, but it remains evident even as the data distribution becomes more homogeneous. By estimating reliability at the sample level and modulating backward distillation accordingly, \ours preserves the benefit of collaboration while substantially reducing the risk that a subset of clients is harmed by the transfer process.

\subsection{Predictive Performance under Negative Transfer Mitigation}

\begin{figure*}[t]
\centering
    \includegraphics[width=0.95\textwidth]{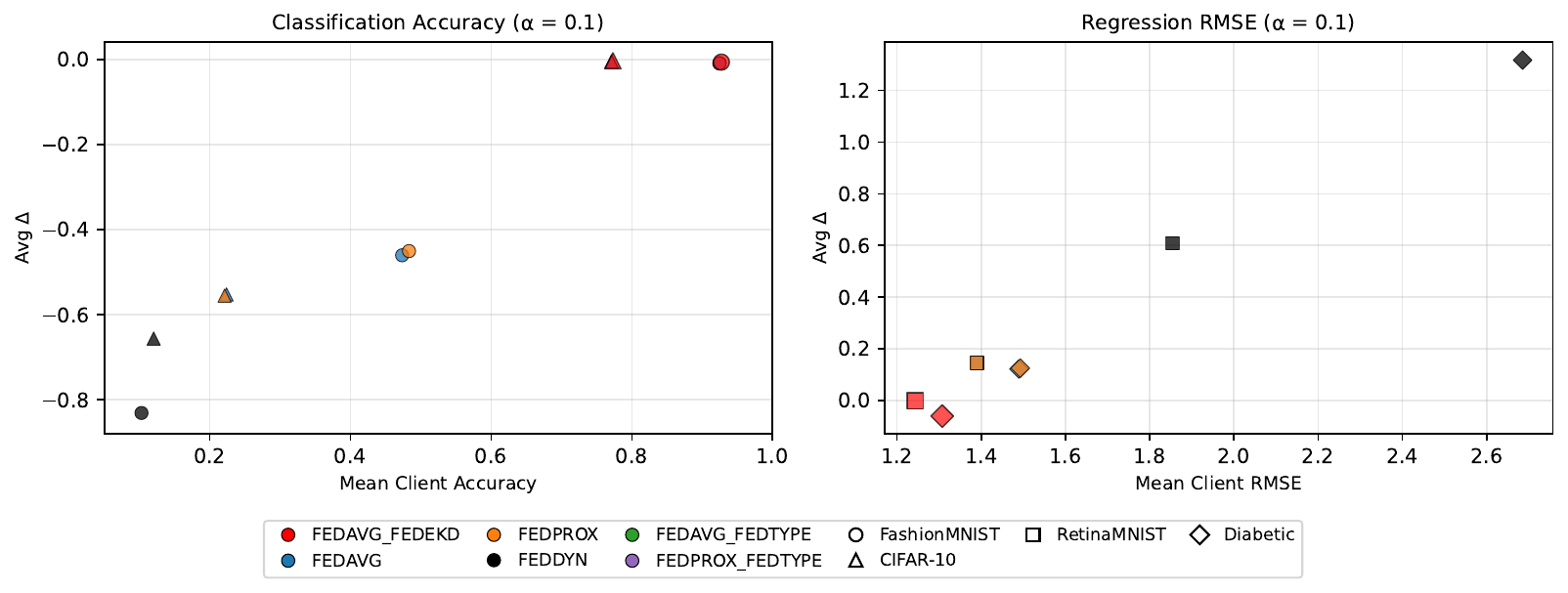}
    \caption{Classification accuracy on FashionMNIST and CIFAR-10 (left) as well as RMSE regression on RetinaMNIST and Diabetic Retinopathy (right) at $\alpha = 0.1$. Full results are reported in Appendix~\ref{sec:acc_n_rmse}.}
    \label{fig:figure2}
\end{figure*}

We further examine whether \ours maintains strong predictive performance in terms of mean client accuracy and RMSE. While these metrics are not the primary optimization objective, they provide an important complementary view of overall model quality.

\paragraph{Classification.}
As shown in Figure~\ref{fig:figure2} (left) and Table~\ref{tab:mean_client_acc_main} in Appendix~\ref{sec:acc_n_rmse}, \ours remains consistently competitive in mean client accuracy across all datasets and heterogeneity levels. In highly heterogeneous settings ($\alpha = 0.1$), \ours achieves the best accuracy on all four datasets, indicating that reliability-aware knowledge transfer does not compromise predictive performance even under severe distribution shifts. At moderate heterogeneity ($\alpha = 0.3$), \ours continues to rank among the top methods, often matching or slightly exceeding FedType-based variants. As data distributions become more homogeneous ($\alpha = 0.5$), the performance gap between methods naturally narrows, yet \ours still maintains near-optimal accuracy across all datasets. Importantly, in the few cases where \ours is not the top-performing method in accuracy, the difference is marginal. This behavior aligns with the design objective of the method: rather than maximizing accuracy alone, \ours prioritizes robustness against negative transfer while preserving high overall performance. The results therefore demonstrate that strong reductions in negative transfer can be achieved without sacrificing predictive accuracy.

\paragraph{Regression.}
A similar pattern is observed in regression tasks. As reported in Figure~\ref{fig:figure2} (right) and Table~\ref{tab:mean_rmse_main}, \ours consistently achieves the lowest mean client RMSE across both RetinaMNIST and Diabetic Retinopathy under all heterogeneity levels. This result is particularly notable because standard FL baselines not only suffer from significant negative transfer (Table~\ref{tab:neg_transfer_regression}) but also exhibit worse predictive performance in absolute terms.

\section{Ablation Studies}
\paragraph{Ablation Analysis of the Gating Mechanism.}

The central methodological component of \ours is the reliability-aware backward gate, which determines how strongly each transferred proxy signal should affect the private model. To isolate the contribution of this gating mechanism, we compare our KD gating with \textit{No Gating} as well as several energy-based gating variants (\textit{entropy-based} energy, \textit{margin-based} energy, \textit{LogSumExp} energy, and \textit{feature-distance} energy) under the same forward-stage setup. All descriptions of these variants can be found in Appendix~\ref{sec:role_en_gating}.
%We first consider \textit{No Gating}, which applies the same backward distillation objective but without any gating, i.e., all samples fully trust the proxy model. 
Table~\ref{tab:ablation_energy_cls} shows \textit{No Gating} consistently suffers from substantially larger negative transfer across all datasets and heterogeneity levels. This confirms that simply introducing backward distillation is insufficient, particularly under high data heterogeneity across clients, where naive full-trust transfer amplifies unreliable and conflicting knowledge. These results demonstrate that effective backward distillation in heterogeneous settings requires a reliability-aware gating mechanism. While some variants, particularly feature-distance energy, achieve competitive performance in certain cases, none consistently match the KL-based gating used in \ours. In particular, alternative energies often exhibit higher variance or degrade under strong heterogeneity, indicating less reliable selection of informative samples.
%Overall, \ours significantly reduces negative transfer, especially under severe heterogeneity ($\alpha=0.1$) and moderate heterogeneity ($\alpha=0.3$), demonstrating the effectiveness of selectively filtering transferred knowledge. This clearly shows that the improvement is not due to the presence of backward distillation alone, but specifically due to reliability-aware gating.

\paragraph{Sensitivity to $\beta$.}

\begin{figure}[t]
    \centering
    \includegraphics[width=0.7\textwidth]{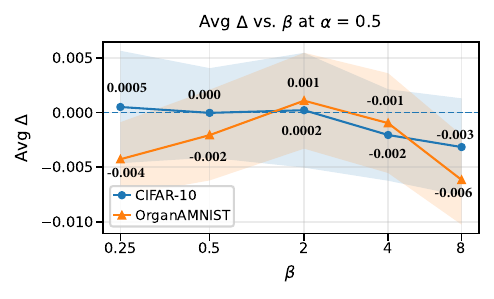}
    \caption{Avg $\Delta$ across $\beta$ values on CIFAR-10 and OrganAMNIST at $\alpha = 0.5$. In practice, $\beta \in [1,2]$ provides a reliable default choice under this setting.}
    \label{fig:figure3}
\end{figure}

We further investigate the impact of the gating sharpness parameter $\beta$, which controls how aggressively unreliable knowledge is suppressed during backward distillation. Figure~\ref{fig:figure3} reports Avg $\Delta$ across $\beta \in \{0.25, 0.5, 2, 4, 8\}$ under mild heterogeneity ($\alpha = 0.5$). 
Across both CIFAR-10 and OrganAMNIST, performance follows a consistent pattern: moderate values of $\beta$ perform best, with $\beta = 2$ achieving the highest Avg $\Delta$. Smaller values (e.g., $\beta = 0.25$) apply nearly uniform weights, resulting in weak gating and limited improvement over naive distillation, while larger values (e.g., $\beta = 4, 8$) lead to overly selective gating that suppresses useful knowledge and degrades performance. 
Interestingly, the optimal $\beta$ remains above the default ($\beta = 1$) even under mild heterogeneity. This suggests that unreliable knowledge persists beyond data heterogeneity and also arises from model asymmetry and proxy aggregation, necessitating non-trivial gating even when client distributions are relatively homogeneous.

%This behavior is further reflected in the average KD weight, which increases monotonically with $\beta$. While higher $\beta$ improves selectivity, it also reduces the effective contribution of the proxy model, eventually approaching under-utilization of shared knowledge. Conversely, very small $\beta$ values fail to sufficiently differentiate between reliable and unreliable signals, diminishing the benefit of reliability-aware gating.

%β→0: FEDEKD tiến gần KD không gating
%β vừa phải: negative transfer giảm mạnh nhất
%β quá lớn: có thể under-distill, accuracy tụt hoặc variance tăng

\paragraph{Sensitivity to $\lambda_{\text{kd}}$.}

We vary $\lambda_{\text{kd}} \in \{0, 0.25, 0.5, 2, 4\}$ while fixing other hyperparameters to examine its impact on negative transfer under mild heterogeneity ($\alpha = 0.5$). As shown in Figure~\ref{fig:figure4} of Appendix~\ref{sec:sensitive_2_lambda}, performance exhibits a clear unimodal pattern, with moderate values $\lambda_{\text{kd}} \in [0.25, 0.5]$ achieving the highest Avg $\Delta$, while larger values lead to substantial degradation. This reflects a trade-off between utilizing proxy knowledge and preserving local learning: small $\lambda_{\text{kd}}$ under-utilizes shared knowledge, whereas large $\lambda_{\text{kd}}$ amplifies the distillation term ($\lambda_{\text{kd}} \cdot w(x)$), effectively overriding the gating mechanism and reintroducing negative transfer. In practice, while the default setting ($\lambda_{\text{kd}} = 1$) provides a stable and reliable choice, we recommend reducing $\lambda_{\text{kd}}$ to $0.5$ or $0.25$ under less severe heterogeneity, as moderate distillation is sufficient and avoids over-amplifying proxy-induced discrepancies.
%We further investigate the effect of the distillation strength parameter $\lambda_{\text{kd}}$, which controls the overall influence of transferred knowledge during backward distillation. We vary $\lambda_{\text{kd}} \in \{0, 0.25, 0.5, 2, 4\}$ while fixing other hyperparameters. Across all settings, performance follows a consistent pattern: moderate values of $\lambda_{\text{kd}}$ yield the best results, with $\lambda_{\text{kd}} = 1$ achieving the strongest reduction in negative transfer. When $\lambda_{\text{kd}}$ is too small, the influence of knowledge transfer becomes negligible, resulting in behavior close to local training. Conversely, large values of $\lambda_{\text{kd}}$ overly amplify transferred knowledge, leading to performance degradation. These results highlight that $\lambda_{\text{kd}}$ controls the global strength of knowledge transfer, while the gating parameter $\beta$ determines its sample-wise selectivity. Together, they provide complementary mechanisms for balancing trust and robustness in heterogeneous FL settings.

\section{Conclusion} \label{sec:diss_n_con}

\paragraph{Limitations.}
Our study has several limitations. First, the experiments are conducted on public benchmarks with simulated cross-silo partitions, which may not fully capture real-world multi-site variation such as differences in data acquisition, population characteristics, annotation quality, and site-specific workflows. Second, while \ours improves client-level robustness, we do not directly evaluate fairness across subpopulations; thus, our results should be interpreted as improving robustness under heterogeneity rather than addressing broader fairness concerns. Third, \ours requires maintaining both private and proxy models and performing bidirectional distillation, which introduces additional computation and synchronization overhead. Finally, the regression energy measures prediction disagreement rather than calibrated predictive uncertainty.

\paragraph{Conclusion and Future Work.}
We introduced \ours, a reliability-aware federated knowledge distillation framework for heterogeneous FL. By converting private-proxy disagreement into sample-wise trust weights, \ours mitigates negative transfer while maintaining strong predictive performance across classification and regression tasks. Future work will extend \ours toward real multi-site deployments, fairness-aware evaluation, adaptive and calibration-aware trust estimation, and more communication-efficient training.

\newpage
\begin{ack}
...
\end{ack}

%\section*{References}
\bibliography{ref}
\bibliographystyle{abbrvnat}

%References follow the acknowledgments in the camera-ready paper. Use unnumbered first-level heading for
%the references. Any choice of citation style is acceptable as long as you are
%consistent. It is permissible to reduce the font size to \verb+small+ (9 point)
%when listing the references.
%Note that the Reference section does not count towards the page limit.
%\medskip

%{
%\small

%[1] Alexander, J.A.\ \& Mozer, M.C.\ (1995) Template-based algorithms for
%connectionist rule extraction. In G.\ Tesauro, D.S.\ Touretzky and T.K.\ Leen
%(eds.), {\it Advances in Neural Information Processing Systems 7},
%pp.\ 609--616. Cambridge, MA: MIT Press.

%[2] Bower, J.M.\ \& Beeman, D.\ (1995) {\it The Book of GENESIS: Exploring
 % Realistic Neural Models with the GEneral NEural SImulation System.}  New York:
%TELOS/Springer--Verlag.

%[3] Hasselmo, M.E., Schnell, E.\ \& Barkai, E.\ (1995) Dynamics of learning and
%recall at excitatory recurrent synapses and cholinergic modulation in rat
%hippocampal region CA3. {\it Journal of Neuroscience} {\bf 15}(7):5249-5262.
%}

%%%%%%%%%%%%%%%%%%%%%%%%%%%%%%%%%%%%%%%%%%%%%%%%%%%%%%%%%%%%
\newpage
\appendix

% start local toc
\startcontents[appendix]

\section*{Contents}
\printcontents[appendix]{}{1}{}
\clearpage

\section{Broader Impact} \label{sec:broader_impact}

\ours introduces a reliability-aware federated knowledge distillation framework that offers a promising approach for advancing healthcare AI under data heterogeneity and privacy constraints. By enabling collaborative learning without requiring direct data sharing, the framework supports compliance with strict data protection regulations while preserving patient confidentiality. At the same time, \ours improves the robustness and stability of models trained across diverse clinical sites, allowing them to better generalize across variations in data acquisition, population characteristics, and annotation quality. By leveraging disagreement between private and proxy models to guide knowledge transfer, \ours enhances the effective use of distributed and heterogeneous medical data without centralization. This capability can expand the applicability of AI systems in healthcare, supporting more reliable diagnostic models and decision support tools across institutions with varying resources. Furthermore, the framework promotes broader participation in collaborative learning settings, enabling institutions with limited data or computational capacity to benefit from shared model improvements. Overall, \ours has the potential to contribute to more scalable, privacy-preserving, and robust healthcare AI systems, supporting improved clinical decision-making and advancing the development of trustworthy machine learning solutions in sensitive real-world environments.

\section{Compute and Environment Configuration} \label{sec:compute_env}

All experiments are conducted on servers equipped with NVIDIA A100 GPUs (80 GB) and AMD EPYC CPUs, running on Linux-based systems with CUDA 11.8. Models are implemented in PyTorch and trained using single-GPU setups with up to 32 CPU cores per run. Additional implementation details, including hyperparameters and training configurations, are provided in the accompanying code repository.

\section{Code and Repository} \label{sec:code}

We provide an implementation of \ours at \textcolor{blue}{\textbf{link}}. This repository includes resources for data preprocessing scripts, baseline methods, and experiment configurations. The codebase is designed to facilitate reproducibility and further research in heterogeneous federated learning. All relevant details and instructions are documented in the repository.

\section{Author Statement}

The authors take full responsibility for the content of this work, including the implementation and evaluation of the proposed method. Any potential issues related to data usage, licensing, or reproducibility will be addressed through updates to the public repository.

\newpage
\section{Algorithm of \ours}

\begin{algorithm}[H]
\caption{One communication round of \ours}
\label{alg:fedekd}
\begin{algorithmic}[1]
\Require Client datasets $\{\mathcal{D}_k\}_{k=1}^K$, private models $\{f_k\}_{k=1}^K$, proxy models $\{g_k\}_{k=1}^K$, $\lambda_{\mathrm{kd}}$, $\beta$
\For{each client $k$ in parallel}
    \State Freeze private model $f_k$
    \State Update proxy model $g_k$ on $\mathcal{D}_k$ by forward KD from $f_k$
    \State Upload proxy model $g_k$ to the server
\EndFor
\State Server aggregates proxy models: $g \leftarrow K^{-1}\sum_{k=1}^{K} g_k$
\State Server broadcasts global proxy $g$ to all clients
\For{each client $k$ in parallel}
    \State Freeze global proxy $g$
    \For{each minibatch $B\subset \mathcal{D}_k$}
        \State Compute private outputs $f_k(x_i)$ and proxy outputs $g(x_i)$ for $x_i\in B$
        \State Compute energy scores $E_i$ and trust weights $w_i$ using Eq.~\ref{eq:trust_weight}
        \State Detach $w_i$ from the computation graph
        \State Update private model $f_k$ using
        $\mathcal{L}_{\mathrm{sup}}+\lambda_{\mathrm{kd}}|B|^{-1}\sum_{x_i\in B}\operatorname{sg}(w_i)\ell_{\mathrm{KD}}(x_i)$
    \EndFor
\EndFor
\end{algorithmic}
\end{algorithm}

\paragraph{Computational and communication cost.}
Let $F_f$ and $B_f$ denote the forward and backward costs of the private model, and let $F_g$ and $B_g$ denote those of the proxy model. For each minibatch, forward proxy distillation costs one frozen private forward pass and one proxy forward-backward update, i.e., $F_f+F_g+B_g$. The backward stage costs one frozen proxy forward pass and one private forward-backward update, i.e., $F_g+F_f+B_f$, plus the energy and gating computation, which is $O(|B|C)$ for classification with $C$ classes and $O(|B|)$ for regression. Therefore, compared with a standard supervised private update, \ours adds one private teacher forward pass, one proxy teacher forward pass, and one lightweight proxy update per communication round. The communication cost is $O(|\theta_g|)$ per client per round because only proxy parameters are uploaded and broadcast, while private parameters $\theta_f$ remain local.

\section{Mechanistic and Stability Properties of Energy-Gated Distillation}
\label{sec:gate_theory}

The following results are intended as stability and mechanistic properties of the batch-normalized logistic gate, rather than as new mathematical inequalities. They clarify how the gate orders samples, bounds trust weights, and modulates per-sample distillation gradients under stop-gradient weighting.

\begin{proposition}[Batch-relative monotonicity]
\label{prop:batch_relative_monotonicity}
Let $B=\{x_i\}_{i=1}^{n}$ be a minibatch with energy scores
$E_i=E(x_i)$. Define the batch mean and standard deviation as
\[
\mu_B=\frac{1}{n}\sum_{j=1}^{n}E_j,
\qquad
s_B=\sqrt{\frac{1}{n}\sum_{j=1}^{n}(E_j-\mu_B)^2}.
\]
For $\epsilon_B>0$ and $\beta>0$, define the normalized energy and trust
weight by
\[
\widetilde{E}_i=\frac{E_i-\mu_B}{s_B+\epsilon_B},
\qquad
w_i=\rho(-\beta\widetilde{E}_i),
\qquad
\rho(t)=\frac{1}{1+\exp(-t)}.
\]
Then for any two samples $x_i,x_j\in B$,
\[
E_i\leq E_j
\quad\Longrightarrow\quad
w_i\geq w_j.
\]
\end{proposition}

\begin{proof}
Since $s_B+\epsilon_B>0$, the batch normalization is order-preserving:
\[
E_i\leq E_j
\quad\Longrightarrow\quad
\frac{E_i-\mu_B}{s_B+\epsilon_B}
\leq
\frac{E_j-\mu_B}{s_B+\epsilon_B}.
\]
Therefore,
\[
\widetilde{E}_i\leq \widetilde{E}_j.
\]
The map $z\mapsto -\beta z$ is monotonically decreasing for $\beta>0$, and
the sigmoid function $\rho(t)$ is monotonically increasing. Hence the
composition $z\mapsto \rho(-\beta z)$ is monotonically decreasing. Thus,
\[
\widetilde{E}_i\leq \widetilde{E}_j
\quad\Longrightarrow\quad
\rho(-\beta\widetilde{E}_i)
\geq
\rho(-\beta\widetilde{E}_j),
\]
which gives
\[
w_i\geq w_j.
\]
\end{proof}

\begin{proposition}[Bounded and non-degenerate trust weights]
\label{prop:nondegenerate_trust_weights}
Let $B=\{x_i\}_{i=1}^{n}$ be a minibatch with energy scores
$E_i=E(x_i)$. Define
\[
\mu_B=\frac{1}{n}\sum_{j=1}^{n}E_j,
\qquad
s_B=\sqrt{\frac{1}{n}\sum_{j=1}^{n}(E_j-\mu_B)^2},
\]
and
\[
\widetilde{E}_i=\frac{E_i-\mu_B}{s_B+\epsilon_B},
\qquad
w_i=\rho(-\beta\widetilde{E}_i),
\qquad
\rho(t)=\frac{1}{1+\exp(-t)},
\]
where $\epsilon_B>0$ and $\beta>0$. If $n=1$, then $w_1=1/2$.
If $n\geq 2$, then for every sample $x_i\in B$,
\[
\rho\!\left(-\beta\sqrt{n-1}\right)
\leq
w_i
\leq
\rho\!\left(\beta\sqrt{n-1}\right).
\]
In particular, $0<w_i<1$ for all samples.
\end{proposition}

\begin{proof}
If $n=1$, then $E_1=\mu_B$ and $s_B=0$, so
$\widetilde{E}_1=0$ and $w_1=\rho(0)=1/2$.

Now assume $n\geq 2$. If $s_B=0$, then all energy scores are equal,
so $\widetilde{E}_i=0$ and $w_i=1/2$ for all $i$, which satisfies the
claim.

It remains to consider the case $s_B>0$. Define
\[
u_i=\frac{E_i-\mu_B}{s_B}.
\]
Then
\[
\sum_{i=1}^{n}u_i=0,
\qquad
\sum_{i=1}^{n}u_i^2=n.
\]
For any fixed $i$, we have
\[
\sum_{j\neq i}u_j=-u_i.
\]
By Cauchy's inequality,
\[
\left(\sum_{j\neq i}u_j\right)^2
\leq
(n-1)\sum_{j\neq i}u_j^2.
\]
Thus,
\[
\sum_{j\neq i}u_j^2
\geq
\frac{u_i^2}{n-1}.
\]
Therefore,
\[
n
=
u_i^2+\sum_{j\neq i}u_j^2
\geq
u_i^2+\frac{u_i^2}{n-1}
=
\frac{n}{n-1}u_i^2.
\]
Hence,
\[
|u_i|\leq \sqrt{n-1}.
\]
Since
\[
|\widetilde{E}_i|
=
\left|
\frac{E_i-\mu_B}{s_B+\epsilon_B}
\right|
=
|u_i|\frac{s_B}{s_B+\epsilon_B}
\leq
|u_i|,
\]
we obtain
\[
|\widetilde{E}_i|\leq \sqrt{n-1}.
\]
Thus,
\[
-\sqrt{n-1}
\leq
\widetilde{E}_i
\leq
\sqrt{n-1}.
\]
Because $z\mapsto \rho(-\beta z)$ is monotonically decreasing for
$\beta>0$,
\[
\rho\!\left(-\beta\sqrt{n-1}\right)
\leq
w_i
\leq
\rho\!\left(\beta\sqrt{n-1}\right).
\]
Finally, since the sigmoid function takes values strictly between $0$
and $1$, we have $0<w_i<1$.
\end{proof}

\begin{proposition}[Variational interpretation of the logistic gate]
\label{prop:logistic_gate_variational}
For any normalized energy $\widetilde{E}\in\mathbb{R}$ and any
$\beta>0$, the trust weight
\[
w^\star=\rho(-\beta\widetilde{E})
=
\frac{1}{1+\exp(\beta\widetilde{E})}
\]
is the unique minimizer over $w\in(0,1)$ of
\[
\phi(w)
=
w\widetilde{E}
+
\frac{1}{\beta}
\left[
w\log w+(1-w)\log(1-w)
\right].
\]
\end{proposition}

\begin{proof}
For $w\in(0,1)$,
\[
\phi'(w)
=
\widetilde{E}
+
\frac{1}{\beta}
\log\frac{w}{1-w},
\]
and
\[
\phi''(w)
=
\frac{1}{\beta w(1-w)}
>0.
\]
Therefore, $\phi$ is strictly convex on $(0,1)$ and has at most one
stationary point. Setting $\phi'(w)=0$ gives
\[
\log\frac{w}{1-w}
=
-\beta\widetilde{E}.
\]
Exponentiating both sides,
\[
\frac{w}{1-w}
=
\exp(-\beta\widetilde{E}).
\]
Solving for $w$ yields
\[
w
=
\frac{1}{1+\exp(\beta\widetilde{E})}
=
\rho(-\beta\widetilde{E}).
\]
Since $\phi$ is strictly convex, this stationary point is the unique
minimizer.
\end{proof}

\begin{proposition}[Detached gating preserves the per-sample distillation direction]
\label{prop:detached_gate_direction}
Let $\theta$ denote the parameters of the private model and let
$\ell_i(\theta)$ be the distillation loss for sample $x_i$. Consider the
detached gated loss
\[
L_i(\theta)
=
\operatorname{sg}(w_i)\ell_i(\theta),
\]
where $\operatorname{sg}(\cdot)$ denotes stop-gradient and $0<w_i<1$.
Then
\[
\nabla_\theta L_i(\theta)
=
w_i\nabla_\theta\ell_i(\theta).
\]
Consequently, for each individual sample, detached gating rescales the
distillation gradient by a positive scalar and cannot reverse its
direction.
\end{proposition}

\begin{proof}
Because gradients are stopped through $w_i$, the trust weight is treated
as a constant with respect to $\theta$ during backpropagation. Therefore,
\[
\nabla_\theta L_i(\theta)
=
\nabla_\theta
\left[
\operatorname{sg}(w_i)\ell_i(\theta)
\right]
=
w_i\nabla_\theta\ell_i(\theta).
\]
Since $0<w_i<1$, the gated per-sample gradient is a positive scalar
multiple of the ungated per-sample gradient. Hence the direction cannot
be reversed for that sample.
\end{proof}

\begin{corollary}[Regression output-level influence]
\label{cor:regression_output_influence}
For regression, let
\[
r_i=f_\theta(x_i)-g(x_i)
\]
and let the distillation loss be
\[
\ell_i=\|r_i\|_2^2.
\]
Under detached gating,
\[
L_i=\operatorname{sg}(w_i)\ell_i.
\]
Then the gradient with respect to the private model output is
\[
\nabla_{f_\theta(x_i)}L_i
=
2w_i r_i.
\]
Therefore, the negative gradient is proportional to
\[
g(x_i)-f_\theta(x_i),
\]
and the gated output-level gradient has no larger magnitude than the
ungated distillation gradient:
\[
\left\|
\nabla_{f_\theta(x_i)}L_i
\right\|_2
\leq
\left\|
\nabla_{f_\theta(x_i)}\ell_i
\right\|_2.
\]
\end{corollary}

\begin{proof}
Since $w_i$ is detached, it is constant with respect to
$f_\theta(x_i)$. Thus,
\[
\nabla_{f_\theta(x_i)}L_i
=
w_i\nabla_{f_\theta(x_i)}
\|f_\theta(x_i)-g(x_i)\|_2^2.
\]
The gradient of the squared error is
\[
\nabla_{f_\theta(x_i)}
\|f_\theta(x_i)-g(x_i)\|_2^2
=
2(f_\theta(x_i)-g(x_i))
=
2r_i.
\]
Hence,
\[
\nabla_{f_\theta(x_i)}L_i=2w_i r_i.
\]
Since $0<w_i<1$,
\[
\left\|
2w_i r_i
\right\|_2
\leq
\left\|
2r_i
\right\|_2,
\]
which proves the magnitude bound. The negative gradient is
\[
-2w_i r_i
=
2w_i(g(x_i)-f_\theta(x_i)),
\]
so the distillation update pulls the private prediction toward the proxy
prediction without reversing the squared-error distillation direction.
\end{proof}

\begin{corollary}[Classification output-level influence]
\label{cor:classification_output_influence}
For classification, let
\[
z_i=f_\theta(x_i),
\qquad
p_i=\mathrm{softmax}(z_i),
\qquad
q_i=\mathrm{softmax}(g(x_i)),
\]
where $q_i$ is treated as fixed during the private-model update. Let the
distillation loss be
\[
\ell_i=\mathrm{KL}(q_i\|p_i).
\]
Under detached gating,
\[
L_i=\operatorname{sg}(w_i)\ell_i.
\]
Then
\[
\nabla_{z_i}L_i
=
w_i(p_i-q_i).
\]
Moreover,
\[
\left\|
\nabla_{z_i}L_i
\right\|_2
\leq
\left\|p_i-q_i\right\|_2
\leq
\sqrt{2}.
\]
\end{corollary}

\begin{proof}
The KL loss can be written as
\[
\ell_i
=
\mathrm{KL}(q_i\|p_i)
=
\sum_{c}q_{ic}\log q_{ic}
-
\sum_{c}q_{ic}\log p_{ic}.
\]
The first term is constant with respect to $z_i$. Since
$p_i=\mathrm{softmax}(z_i)$, the gradient of the cross-entropy term with
respect to logits is
\[
\nabla_{z_i}\ell_i
=
p_i-q_i.
\]
Because $w_i$ is detached,
\[
\nabla_{z_i}L_i
=
w_i\nabla_{z_i}\ell_i
=
w_i(p_i-q_i).
\]
Since $0<w_i<1$,
\[
\left\|
\nabla_{z_i}L_i
\right\|_2
=
w_i\|p_i-q_i\|_2
\leq
\|p_i-q_i\|_2.
\]

It remains to show that $\|p_i-q_i\|_2\leq\sqrt{2}$. Let
$v=p_i-q_i$. Since both $p_i$ and $q_i$ are probability vectors,
$\sum_c v_c=0$. Let $v^+_c=\max(v_c,0)$ and
$v^-_c=\max(-v_c,0)$. Then
\[
\|v^+\|_1=\|v^-\|_1=a
\]
for some $a\leq 1$. Therefore,
\[
\|v\|_2^2
=
\|v^+\|_2^2+\|v^-\|_2^2
\leq
\|v^+\|_1^2+\|v^-\|_1^2
=
2a^2
\leq
2.
\]
Thus,
\[
\|p_i-q_i\|_2\leq \sqrt{2}.
\]
\end{proof}

\paragraph{Remark.} When the full objective includes the factor $\lambda_{\mathrm{kd}}$, the
distillation gradients in Corollaries~\ref{cor:regression_output_influence}
and~\ref{cor:classification_output_influence} are scaled by
$\lambda_{\mathrm{kd}}$.

These results do not claim that the gate theoretically guarantees a
reduction in final negative transfer. Rather, they show that the proposed
batch-normalized logistic gate provides a principled relative trust
assignment and that, under stop-gradient gating, it modulates the
per-sample distillation influence without reversing the underlying
distillation direction.

\section{Data Licence \& Data Statistics} \label{app:dat_stats}

\subsection{License and Ethics} 
All datasets used in this work are publicly available and are used in accordance with their respective licenses and terms of use. 
The medical imaging datasets (OCTMNIST, OrganAMNIST, and RetinaMNIST) are derived from MedMNIST (\url{https://medmnist.com/}) and are released under the Creative Commons Attribution (CC BY 4.0) license. 
The Diabetic Retinopathy dataset is obtained via the MedIMeta benchmark (\url{https://www.woerner.eu/projects/medimeta/}), which aggregates multiple medical imaging datasets for research use. 
FashionMNIST is distributed under the MIT License, while CIFAR-10 is released for research purposes. 
All datasets consist of de-identified images and annotations. We properly cite all original data sources and comply with their usage restrictions.

\subsection{Data Statistics}
We evaluate \ours on six public benchmarks spanning
standard vision datasets and diverse medical imaging tasks.
These datasets cover heterogeneous data modalities, multi-class classification,
and ordinal regression tasks, with dataset sizes ranging from thousands to
hundreds of thousands of samples.

\subsection*{Datasets Overview}

\begin{table*}[h] 
\centering
\resizebox{0.9\textwidth}{!}{
\begin{tabular}{l l l c c r r}
\toprule
\textbf{Dataset} 
& \textbf{Data Modality} 
& \textbf{Task (\# Classes/Labels)} 
& \textbf{Image Size} 
& \textbf{Channels} 
& \textbf{\# Samples} 
& \textbf{\# Train / Val / Test} \\
\midrule

FashionMNIST 
& Apparel Images 
& Multi-Class (10) 
& 28$\times$28 
& 1 
& 70,000 
& 60,000 / -- / 10,000 \\

CIFAR-10 
& Natural Images 
& Multi-Class (10) 
& 32$\times$32 
& 3 
& 60,000 
& 50,000 / -- / 10,000 \\

OCTMNIST 
& Retinal OCT
& Multi-Class (4) 
& 28$\times$28
& 1
& 109,309
& 97,477 / 10,832 / 1,000 \\

OrganAMNIST
& Abdominal CT
& Multi-Class (11) 
& 28$\times$28
& 1 
& 58,830
& 34,561 / 6,491 / 17,778 \\

RetinaMNIST 
& Fundus Camera 
& Ordinal Regression (5) 
& 28$\times$28 
& 3 
& 1,600 
& 1,080 / 120 / 400 \\

Diabetic Retinopathy
& Diabetic Retinopathy
& Ordinal regression (5)
& 224$\times$224
& 3 
& 2,000
& 1,400 / 300 / 300 \\

\bottomrule
\end{tabular}
}
\caption{Summary of datasets used in our experiments. We include standard vision benchmarks and diverse medical imaging datasets spanning classification and regression tasks.}
\end{table*}

\subsection*{Federated Split and Heterogeneity Design}

For all datasets, we simulate a cross-silo FL setting with $K=6$ clients. Data are first partitioned across clients using a Dirichlet distribution with
concentration parameter $\alpha\in\{0.1,0.3,0.5\}$, where smaller values induce
stronger heterogeneity. Unless otherwise specified, the same split procedure is
applied for each value of $\alpha$.

For classification tasks, this produces label skew across clients. For regression tasks (RetinaMNIST and Diabetic Retinopathy), we induce covariate shift by clustering feature representations into $B=5$ bins using K-means and applying the Dirichlet partition over these clusters.

After Dirichlet partitioning, each client's local dataset is further split into training, validation, and test subsets with proportions 60\%, 20\%, and 20\%, respectively. The training split is used for model optimization, the test split is used for evaluating negative transfer and robustness metrics, and the validation split is reserved for baselines that require an additional hold-out set (e.g., FedType in this work). Before splitting, each client’s data are randomly permuted to avoid ordering bias.

\paragraph{RetinaMNIST and Diabetic Retinopathy as Regression.}

RetinaMNIST and Diabetic Retinopathy are originally defined as ordinal classification problems with five ordered grades. In our work, we treat both datasets as regression tasks by modeling the ordered labels as numeric targets. This formulation enables evaluation of \ours under a regression-style setting while preserving the ordinal structure of disease severity levels.

\section{Study Design} \label{app:study_design}

\paragraph{Federated Setup.}
We simulate a cross-silo federated learning system with $K=6$ clients. All experiments are repeated over $10$ independent random seeds, and we report the mean and standard deviation across runs.

\paragraph{Controlled Heterogeneity.}
To rigorously evaluate \ours under realistic non-IID conditions, we introduce two controlled sources of heterogeneity: data heterogeneity and model asymmetry between private and proxy networks within each client.

\subsection*{Data Heterogeneity (Client-Level Distribution Shift)}

Heterogeneity is introduced during the initial Dirichlet partitioning stage, affecting the entire local dataset of each client. Smaller values of $\alpha$ induce stronger heterogeneity.

\begin{itemize}
    \item \textbf{Classification Tasks.}
    For binary and multi-class classification datasets, we apply Dirichlet label skew. For each class $c$, samples belonging to class $c$ are distributed across clients according to proportions drawn from a Dirichlet distribution. This produces label imbalance and heterogeneous class priors across clients.

    \item \textbf{Regression Tasks.}
    For regression tasks, covariate shift is induced by applying K-means clustering (B=5) directly on the input feature vectors (i.e., flattened images after preprocessing). No pretrained feature extractor is used; clustering is performed in the original input space.
\end{itemize}

%Each client trains its model using only its own local training data, and evaluation is performed on the corresponding local test split.
As a result, both training and evaluation are performed under heterogeneous client-specific data distributions.

\subsection*{Model Asymmetry (Private–Proxy Architecture)}

Model heterogeneity in \ours arises from the asymmetric design between private and proxy models rather than differences across clients.

Each client maintains two models: a high-capacity \textit{private model} for local learning and a lightweight \textit{proxy model} for cross-client knowledge exchange. While all clients share the same model architectures, heterogeneity is introduced through the distinct roles and capacities of these two models.

\subsubsection*{Classification Models}

For classification tasks, we use the following architectures:

\begin{itemize}
    \item \textbf{Private Model (PrivateCNN).}  
    A higher-capacity convolutional neural network with three convolutional layers (64 → 128 → 128 channels), followed by a fully connected layer (256 units) and a task-specific output head. This model captures rich feature representations and serves as the primary learner.

    \item \textbf{Proxy Model (ProxyCNN).}  
    A lightweight convolutional model with two convolutional layers (32 → 64 channels) and a smaller feature dimension (128 units). This model is used for communication and aggregation across clients.
\end{itemize}

Both models share the same input space but differ in representational capacity, creating an asymmetric knowledge transfer setting. This asymmetry is central to \ours, as it introduces potential discrepancies between models that motivate the need for reliability-aware knowledge transfer.

\subsubsection*{Regression Models}

For regression tasks, we use the same architectural design, with the output layer modified to produce a scalar prediction. The private model retains higher representational capacity, while the proxy model remains lightweight, preserving the asymmetric design across tasks.

\section{Worst $\Delta$} \label{worst_delta}

\begin{table}[H]
\caption{
Worst-case negative transfer (Worst $\Delta$) for classification tasks measured (higher is better). 
\ours consistently delivers the strongest worst-case robustness, 
substantially reducing the severity of negative transfer in the most vulnerable clients across all settings.
}
\label{tab:neg_transfer_worst}
\centering
\small
\setlength{\tabcolsep}{3pt}
\resizebox{\textwidth}{!}{
\begin{tabular}{l|l|ccc|ccc|ccc|ccc}
\toprule
\textbf{\shortstack{Agg.\\Method}} & \textbf{Model}
& \multicolumn{3}{c}{\textbf{FashionMNIST}}
& \multicolumn{3}{c}{\textbf{CIFAR-10}}
& \multicolumn{3}{c}{\textbf{OCTMNIST}}
& \multicolumn{3}{c}{\textbf{OrganAMNIST}} \\
\cmidrule(lr){3-5} \cmidrule(lr){6-8} \cmidrule(lr){9-11}
\cmidrule(lr){12-14}
& & $\alpha=0.1$ & $\alpha=0.3$ & $\alpha=0.5$
  & $\alpha=0.1$ & $\alpha=0.3$ & $\alpha=0.5$
  & $\alpha=0.1$ & $\alpha=0.3$ & $\alpha=0.5$
  & $\alpha=0.1$ & $\alpha=0.3$ & $\alpha=0.5$ \\
\midrule

FedDyn & -
& $-0.9786_{\pm 0.0089}$ & $-0.9299_{\pm 0.0205}$ & $-0.8861_{\pm 0.0302}$
& $-0.8820_{\pm 0.0668}$ & $-0.7355_{\pm 0.0547}$ & $-0.6397_{\pm 0.0363}$
& $-0.9888_{\pm 0.0166}$ & $-0.9380_{\pm 0.0383}$ & $-0.8566_{\pm 0.0788}$
& $-0.9818_{\pm 0.0112}$ & $-0.9381_{\pm 0.0174}$ & $-0.9071_{\pm 0.0261}$ \\

\midrule

FedProx & -
& $-0.7963_{\pm 0.0913}$ & $-0.4734_{\pm 0.1491}$ & $-0.3036_{\pm 0.0680}$
& $-0.7486_{\pm 0.0831}$ & $-0.5013_{\pm 0.0791}$ & $-0.4022_{\pm 0.0676}$
& $-0.9661_{\pm 0.0479}$ & $-0.6980_{\pm 0.2349}$ & $-0.5301_{\pm 0.1919}$
& $-0.7944_{\pm 0.1156}$ & $-0.5128_{\pm 0.1160}$ & $-0.3799_{\pm 0.0673}$ \\

FedProx & FedType
& $-0.0356_{\pm 0.0347}$ & $-0.0436_{\pm 0.0256}$ & \underline{$-0.0401_{\pm 0.0120}$}
& $-0.0606_{\pm 0.0308}$ & $-0.0472_{\pm 0.0169}$ & $-0.0226_{\pm 0.0127}$
& $-0.0843_{\pm 0.0768}$ & \underline{$-0.1038_{\pm 0.0835}$} & $-0.1158_{\pm 0.1023}$
& $-0.0604_{\pm 0.0413}$ & $-0.0449_{\pm 0.0377}$ & $-0.0307_{\pm 0.0326}$ \\

\midrule

FedAvg & -
& $-0.7704_{\pm 0.1353}$ & $-0.4832_{\pm 0.1397}$ & $-0.3170_{\pm 0.0687}$
& $-0.7279_{\pm 0.0590}$ & $-0.5099_{\pm 0.0803}$ & $-0.4122_{\pm 0.0569}$
& $-0.9673_{\pm 0.0471}$ & $-0.7240_{\pm 0.2145}$ & $-0.5488_{\pm 0.1809}$
& $-0.8155_{\pm 0.0966}$ & $-0.5445_{\pm 0.1125}$ & $-0.3786_{\pm 0.0711}$ \\

FedAvg & FedType
& \underline{$-0.0356_{\pm 0.0357}$} & \underline{$-0.0422_{\pm 0.0276}$} & $-0.0408_{\pm 0.0154}$
& \underline{$-0.0533_{\pm 0.0306}$} & \underline{$-0.0450_{\pm 0.0210}$} & \underline{$-0.0224_{\pm 0.0133}$}
& \underline{$-0.0725_{\pm 0.0469}$} & $-0.1085_{\pm 0.0863}$ & \underline{$-0.1117_{\pm 0.0937}$}
& \underline{$-0.0589_{\pm 0.0408}$} & \underline{$-0.0394_{\pm 0.0401}$} & \underline{$-0.0256_{\pm 0.0225}$} \\

\rowcolor{gray!12}
FedAvg & \textbf{\ours}
& $\mathbf{-0.0213}_{\pm 0.0138}$ & $\mathbf{-0.0274}_{\pm 0.0152}$ & $\mathbf{-0.0165}_{\pm 0.0130}$
& $\mathbf{-0.0312}_{\pm 0.0209}$ & $\mathbf{-0.0196}_{\pm 0.0077}$ & $\mathbf{-0.0217}_{\pm 0.0118}$
& $\mathbf{-0.0295}_{\pm 0.0325}$ & $\mathbf{-0.0226}_{\pm 0.0166}$ & $\mathbf{-0.0258}_{\pm 0.0183}$
& $\mathbf{-0.0281}_{\pm 0.0088}$ & $\mathbf{-0.0192}_{\pm 0.0131}$ & $\mathbf{-0.0186}_{\pm 0.0126}$ \\

\bottomrule
\end{tabular}
}
\end{table}

\begin{table}[H]
\caption{
Worst-case negative transfer (Worst $\Delta$) for regression tasks measured (lower is better). 
\ours consistently minimizes worst-case degradation, showing strong robustness against extreme negative transfer.
}
\label{tab:neg_transfer_worst_reg}
\centering
\small
\setlength{\tabcolsep}{4pt}
\resizebox{\textwidth}{!}{
\begin{tabular}{l|l|ccc|ccc}
\toprule
\textbf{\shortstack{Agg.\\Method}} & \textbf{Model}
& \multicolumn{3}{c}{\textbf{RetinaMNIST}}
& \multicolumn{3}{c}{\textbf{Diabetic Retinopathy}} \\
\cmidrule(lr){3-5} \cmidrule(lr){6-8}
& & $\alpha=0.1$ & $\alpha=0.3$ & $\alpha=0.5$
  & $\alpha=0.1$ & $\alpha=0.3$ & $\alpha=0.5$ \\
\midrule

FedDyn & -
& $1.3692_{\pm 0.5883}$ & $1.5318_{\pm 0.4742}$ & $1.2250_{\pm 0.2285}$
& $2.2816_{\pm 2.6428}$ & $1.4682_{\pm 0.4773}$ & $1.2619_{\pm 0.3380}$ \\

\midrule

FedProx & -
& $0.3950_{\pm 0.1353}$ & \underline{$0.3467_{\pm 0.1374}$} & $0.2990_{\pm 0.0752}$
& $0.5142_{\pm 0.1458}$ & $0.4311_{\pm 0.1383}$ & $0.4039_{\pm 0.0820}$ \\

\midrule

FedAvg & -
& $\underline{0.3917_{\pm 0.1336}}$ & $0.3480_{\pm 0.1419}$ & $\underline{0.2968_{\pm 0.0665}}$
& $\underline{0.5074_{\pm 0.1478}}$ & \underline{$0.4243_{\pm 0.1285}$} & \underline{$0.3922_{\pm 0.0695}$} \\

\rowcolor{gray!12}
FedAvg & \textbf{\ours}
& $\mathbf{0.1663}_{\pm 0.0945}$ & $\mathbf{0.1783}_{\pm 0.1003}$ & $\mathbf{0.1384}_{\pm 0.0703}$
& $\mathbf{0.2230}_{\pm 0.2000}$ & $\mathbf{0.2072}_{\pm 0.1259}$ & $\mathbf{0.2080}_{\pm 0.1079}$ \\

\bottomrule
\end{tabular}
}
\end{table}

\section{P10 and P90} \label{Pp}

\begin{table}[H]
\caption{
Measurement of lower-tail negative transfer (P10 $\Delta$) for classification tasks (higher is better). 
\ours consistently improves lower-tail client robustness, showing that the reduction in negative transfer also extends to clients near the bottom of the performance distribution.
}
\label{tab:neg_transfer_p10}
\centering
\small
\setlength{\tabcolsep}{3pt}
\resizebox{\textwidth}{!}{
\begin{tabular}{l|l|ccc|ccc|ccc|ccc}
\toprule
\textbf{\shortstack{Agg.\\Method}} & \textbf{Model}
& \multicolumn{3}{c}{\textbf{FashionMNIST}}
& \multicolumn{3}{c}{\textbf{CIFAR-10}}
& \multicolumn{3}{c}{\textbf{OCTMNIST}}
& \multicolumn{3}{c}{\textbf{OrganAMNIST}} \\
\cmidrule(lr){3-5} \cmidrule(lr){6-8} \cmidrule(lr){9-11}
\cmidrule(lr){12-14}
& & $\alpha=0.1$ & $\alpha=0.3$ & $\alpha=0.5$
  & $\alpha=0.1$ & $\alpha=0.3$ & $\alpha=0.5$
  & $\alpha=0.1$ & $\alpha=0.3$ & $\alpha=0.5$
  & $\alpha=0.1$ & $\alpha=0.3$ & $\alpha=0.5$ \\
\midrule

FedDyn & -
& $-0.9643_{\pm 0.0105}$ & $-0.9155_{\pm 0.0234}$ & $-0.8671_{\pm 0.0237}$
& $-0.8334_{\pm 0.0618}$ & $-0.6924_{\pm 0.0474}$ & $-0.6221_{\pm 0.0340}$
& $-0.9607_{\pm 0.0323}$ & $-0.8589_{\pm 0.0588}$ & $-0.7769_{\pm 0.0730}$
& $-0.9709_{\pm 0.0124}$ & $-0.9291_{\pm 0.0172}$ & $-0.8955_{\pm 0.0227}$ \\

\midrule

FedProx & -
& $-0.7016_{\pm 0.0852}$ & $-0.4220_{\pm 0.1321}$ & $-0.2599_{\pm 0.0564}$
& $-0.7173_{\pm 0.0722}$ & $-0.4732_{\pm 0.0677}$ & $-0.3736_{\pm 0.0529}$
& $-0.8866_{\pm 0.0832}$ & $-0.5540_{\pm 0.2042}$ & $-0.3796_{\pm 0.1078}$
& $-0.7348_{\pm 0.1044}$ & $-0.4622_{\pm 0.0925}$ & $-0.3272_{\pm 0.0536}$ \\

FedProx & FedType
& $-0.0252_{\pm 0.0250}$ & $-0.0322_{\pm 0.0193}$ & \underline{$-0.0330_{\pm 0.0115}$}
& $-0.0436_{\pm 0.0241}$ & $-0.0340_{\pm 0.0149}$ & $-0.0168_{\pm 0.0085}$
& $-0.0589_{\pm 0.0434}$ & \underline{$-0.0721_{\pm 0.0517}$} & $-0.0840_{\pm 0.0565}$
& $-0.0415_{\pm 0.0256}$ & $-0.0270_{\pm 0.0213}$ & \underline{$-0.0213_{\pm 0.0183}$} \\

\midrule

FedAvg & -
& $-0.6859_{\pm 0.1080}$ & $-0.4279_{\pm 0.1250}$ & $-0.2802_{\pm 0.0512}$
& $-0.7097_{\pm 0.0611}$ & $-0.4765_{\pm 0.0641}$ & $-0.3760_{\pm 0.0420}$
& $-0.8935_{\pm 0.0837}$ & $-0.5701_{\pm 0.1961}$ & $-0.3999_{\pm 0.1023}$
& $-0.7576_{\pm 0.1004}$ & $-0.4880_{\pm 0.0932}$ & $-0.3357_{\pm 0.0606}$ \\

FedAvg & FedType
& \underline{$-0.0242_{\pm 0.0229}$} & \underline{$-0.0302_{\pm 0.0202}$} & $-0.0339_{\pm 0.0133}$
& \underline{$-0.0372_{\pm 0.0274}$} & \underline{$-0.0285_{\pm 0.0131}$} & \underline{$-0.0154_{\pm 0.0081}$}
& \underline{$-0.0554_{\pm 0.0287}$} & $-0.0777_{\pm 0.0531}$ & \underline{$-0.0806_{\pm 0.0519}$}
& \underline{$-0.0385_{\pm 0.0247}$} & \underline{$-0.0239_{\pm 0.0230}$} & $-0.0179_{\pm 0.0142}$ \\

\rowcolor{gray!12}
FedAvg & \textbf{\ours}
& $\mathbf{-0.0181}_{\pm 0.0119}$ & $\mathbf{-0.0188}_{\pm 0.0097}$ & $\mathbf{-0.0095}_{\pm 0.0068}$
& $\mathbf{-0.0224}_{\pm 0.0134}$ & $\mathbf{-0.0147}_{\pm 0.0067}$ & $\mathbf{-0.0151}_{\pm 0.0070}$
& $\mathbf{-0.0177}_{\pm 0.0182}$ & $\mathbf{-0.0158}_{\pm 0.0126}$ & $\mathbf{-0.0188}_{\pm 0.0118}$
& $\mathbf{-0.0233}_{\pm 0.0082}$ & $\mathbf{-0.0137}_{\pm 0.0081}$ & $\mathbf{-0.0141}_{\pm 0.0102}$ \\

\bottomrule
\end{tabular}
}
\end{table}

\begin{table}[H]
\caption{
Measurement of upper-tail negative transfer (P90 $\Delta$) for regression tasks (lower is better). 
\ours consistently achieves the lowest P90 $\Delta$, indicating strong robustness in high-error clients and effective mitigation of severe negative transfer.
}
\label{tab:p90_delta_reg}
\centering
\small
\setlength{\tabcolsep}{4pt}
\resizebox{\textwidth}{!}{
\begin{tabular}{l|ccc|ccc}
\toprule
\textbf{Method}
& \multicolumn{3}{c}{\textbf{RetinaMNIST}}
& \multicolumn{3}{c}{\textbf{Diabetic Retinopathy}} \\
\cmidrule(lr){2-4} \cmidrule(lr){5-7}
& $\alpha=0.1$ & $\alpha=0.3$ & $\alpha=0.5$
& $\alpha=0.1$ & $\alpha=0.3$ & $\alpha=0.5$ \\
\midrule

FedDyn
& $1.1817_{\pm 0.4682}$ & $1.3385_{\pm 0.3937}$ & $1.0808_{\pm 0.1915}$
& $2.0140_{\pm 2.5796}$ & $1.2947_{\pm 0.3894}$ & $1.1125_{\pm 0.2936}$ \\

FedProx
& $0.3237_{\pm 0.1144}$ & \underline{$0.3038_{\pm 0.1097}$} & \underline{$0.2624_{\pm 0.0665}$}
& $0.4231_{\pm 0.1342}$ & $0.3686_{\pm 0.1087}$ & $0.3458_{\pm 0.0902}$ \\

FedAvg
& \underline{$0.3231_{\pm 0.1136}$} & $0.3028_{\pm 0.1108}$ & $0.2607_{\pm 0.0638}$
& \underline{$0.4217_{\pm 0.1375}$} & \underline{$0.3658_{\pm 0.1048}$} & \underline{$0.3413_{\pm 0.0867}$} \\

\rowcolor{gray!12}
\textbf{\ours}
& $\mathbf{0.1151}_{\pm 0.0677}$ & $\mathbf{0.1292}_{\pm 0.0749}$ & $\mathbf{0.0979}_{\pm 0.0473}$
& $\mathbf{0.1680}_{\pm 0.1369}$ & $\mathbf{0.1551}_{\pm 0.0976}$ & $\mathbf{0.1519}_{\pm 0.1019}$ \\

\bottomrule
\end{tabular}
}
\end{table}

\section{Accuracy \& RMSE} \label{sec:acc_n_rmse}

\begin{table}[H]
\caption{
Mean client accuracy for classification tasks (higher is better). 
Although \ours is designed primarily to reduce negative transfer rather than to optimize accuracy alone, it remains consistently top-tier in mean client accuracy across datasets and heterogeneity levels.
}
\label{tab:mean_client_acc_main}

\centering
\small
\setlength{\tabcolsep}{3pt}
\resizebox{\textwidth}{!}{
\begin{tabular}{c|c|ccc|ccc|ccc|ccc}
\hline
\textbf{\shortstack{Agg.\\Method}} & \textbf{Model}
& \multicolumn{3}{c|}{\textbf{FashionMNIST}}
& \multicolumn{3}{c|}{\textbf{CIFAR-10}}
& \multicolumn{3}{c|}{\textbf{OCTMNIST}}
& \multicolumn{3}{c}{\textbf{OrganAMNIST}} \\

\cline{3-14}
& & $\alpha=0.1$ & $\alpha=0.3$ & $\alpha=0.5$
  & $\alpha=0.1$ & $\alpha=0.3$ & $\alpha=0.5$
  & $\alpha=0.1$ & $\alpha=0.3$ & $\alpha=0.5$
  & $\alpha=0.1$ & $\alpha=0.3$ & $\alpha=0.5$ \\

\hline

FedDyn & -
& $0.1029_{\pm 0.0422}$ & $0.0915_{\pm 0.0119}$ & $0.0960_{\pm 0.0141}$
& $0.1203_{\pm 0.0360}$ & $0.0928_{\pm 0.0107}$ & $0.0987_{\pm 0.0112}$
& $0.2730_{\pm 0.1068}$ & $0.3426_{\pm 0.0759}$ & $0.3446_{\pm 0.0740}$
& $0.1064_{\pm 0.0599}$ & $0.0638_{\pm 0.0109}$ & $0.0632_{\pm 0.0178}$ \\

FedProx & -
& $0.4832_{\pm 0.0668}$ & $0.6381_{\pm 0.0605}$ & $0.6805_{\pm 0.0353}$
& $0.2210_{\pm 0.0725}$ & $0.3187_{\pm 0.0257}$ & $0.3484_{\pm 0.0106}$
& $0.4069_{\pm 0.1013}$ & $0.5702_{\pm 0.1271}$ & $0.6327_{\pm 0.0381}$
& $0.4068_{\pm 0.1088}$ & $0.5829_{\pm 0.0547}$ & $0.6425_{\pm 0.0387}$ \\

FedAvg & -
& $0.4734_{\pm 0.0807}$ & $0.6304_{\pm 0.0588}$ & $0.6706_{\pm 0.0339}$
& $0.2238_{\pm 0.0750}$ & $0.3127_{\pm 0.0269}$ & $0.3432_{\pm 0.0115}$
& $0.4014_{\pm 0.1023}$ & $0.5664_{\pm 0.1289}$ & $0.6235_{\pm 0.0385}$
& $0.3900_{\pm 0.1082}$ & $0.5697_{\pm 0.0584}$ & $0.6304_{\pm 0.0430}$ \\

FedAvg & FedType
& \underline{$0.9251_{\pm 0.0241}$} & $\underline{0.8724_{\pm 0.0287}}$ & $\underline{0.8496_{\pm 0.0196}}$
& $\underline{0.7714_{\pm 0.0393}}$ & \underline{$0.6644_{\pm 0.0285}$} & $\mathbf{0.6303}_{\pm 0.0236}$
& $\underline{0.8811_{\pm 0.0376}}$ & $\underline{0.8383_{\pm 0.0348}}$ & $\underline{0.7930_{\pm 0.0547}}$
& $\underline{0.9266_{\pm 0.0133}}$ & $\mathbf{0.8876}_{\pm 0.0169}$ & $\mathbf{0.8722}_{\pm 0.0126}$ \\

FedProx & FedType
& $0.9242_{\pm 0.0247}$ & $0.8719_{\pm 0.0277}$ & $0.8481_{\pm 0.0202}$
& $0.7709_{\pm 0.0356}$ & $0.6625_{\pm 0.0264}$ & \underline{$0.6283_{\pm 0.0272}$}
& $0.8794_{\pm 0.0409}$ & $0.8371_{\pm 0.0355}$ & $0.7917_{\pm 0.0551}$
& $0.9260_{\pm 0.0125}$ & $0.8860_{\pm 0.0160}$ & $0.8696_{\pm 0.0128}$ \\

\rowcolor{gray!12}
FedAvg & \textbf{\ours}
& $\mathbf{0.9273}_{\pm 0.0192}$ & $\mathbf{0.8813}_{\pm 0.0232}$ & $\mathbf{0.8657}_{\pm 0.0161}$
& $\mathbf{0.7730}_{\pm 0.0380}$ & $\mathbf{0.6649}_{\pm 0.0289}$ & $0.6256_{\pm 0.0243}$
& $\mathbf{0.9018}_{\pm 0.0323}$ & $\mathbf{0.8655}_{\pm 0.0296}$ & $\mathbf{0.8212}_{\pm 0.0429}$
& $\mathbf{0.9303}_{\pm 0.0105}$ & \underline{$0.8869_{\pm 0.0114}$} & $\underline{0.8688_{\pm 0.0127}}$ \\

\hline
\end{tabular}
}
\end{table}

\begin{table}[H]
\caption{
Mean client RMSE for regression tasks (lower is better). 
While our primary objective is to mitigate negative transfer, \ours also achieves the lowest RMSE across both regression benchmarks under all heterogeneity levels.
}
\label{tab:mean_rmse_main}

\centering
\small
\setlength{\tabcolsep}{3pt}
\resizebox{\textwidth}{!}{
\begin{tabular}{c|c|ccc|ccc}
\hline
\textbf{\shortstack{Agg.\\Method}} & \textbf{Model}
& \multicolumn{3}{c|}{\textbf{RetinaMNIST}}
& \multicolumn{3}{c}{\textbf{Diabetic Retinopathy}} \\

\cline{3-8}
& & $\alpha=0.1$ & $\alpha=0.3$ & $\alpha=0.5$
  & $\alpha=0.1$ & $\alpha=0.3$ & $\alpha=0.5$ \\

\hline

FedDyn & -
& $1.8538_{\pm 0.2255}$ & $2.0796_{\pm 0.2472}$ & $2.0110_{\pm 0.1302}$
& $2.6852_{\pm 2.3840}$ & $2.0728_{\pm 0.3605}$ & $1.9893_{\pm 0.1440}$ \\

FedProx & -
& $\underline{1.3892_{\pm 0.0945}}$ & $1.3835_{\pm 0.0815}$ & $1.3947_{\pm 0.0456}$
& $1.4925_{\pm 0.1789}$ & $1.4914_{\pm 0.1018}$ & \underline{$1.4840_{\pm 0.0874}$} \\

FedAvg & -
& $1.3907_{\pm 0.0962}$ & $\underline{1.3793_{\pm 0.0846}}$ & $\underline{1.3934_{\pm 0.0439}}$
& \underline{$1.4894_{\pm 0.1804}$} & \underline{$1.4896_{\pm 0.1006}$} & $1.4842_{\pm 0.0861}$ \\

\rowcolor{gray!12}
FedAvg & \textbf{\ours}
& $\mathbf{1.2428}_{\pm 0.0868}$ & $\mathbf{1.2360}_{\pm 0.0781}$ & $\mathbf{1.2625}_{\pm 0.0480}$
& $\mathbf{1.3072}_{\pm 0.1773}$ & $\mathbf{1.3164}_{\pm 0.0896}$ & $\mathbf{1.3276}_{\pm 0.1061}$ \\

\hline
\end{tabular}
}
\end{table}

\section{Role of Energy-based Gating} \label{sec:role_en_gating}

To evaluate the role of the energy function in reliability-aware gating, we consider several alternative formulations that capture different notions of uncertainty and disagreement. In all cases, the energy score is first converted into a batch-normalized
relative energy $\widetilde{E}_i=(E_i-\mu_B)/(s_B+\epsilon_B)$ and then mapped
to a trust weight $w_i=\rho(-\beta\widetilde{E}_i)$, following
Eq.~\ref{eq:trust_weight}. Thus, samples with higher relative energy within
the minibatch receive lower trust weights.

\begin{table}[H]
\caption{
Ablation study of backward distillation strategies under different energy functions. 
\textit{No Gating} applies ungated full-trust backward distillation.
Alternative energy functions (\textit{entropy}, \textit{margin}, \textit{log-sum-exp}, and \textit{feature-based} distances) are evaluated under the same training protocol. 
\ours consistently achieves the best or second-best in reducing negative transfer and maintaining stability.
}
\label{tab:ablation_energy_cls}

\centering
\small
\setlength{\tabcolsep}{3pt}
\resizebox{\textwidth}{!}{
\begin{tabular}{c|c|ccc|ccc|ccc|ccc}
\hline
\textbf{\shortstack{Agg.\\Method}} & \textbf{Model}
& \multicolumn{3}{c|}{\textbf{FashionMNIST}}
& \multicolumn{3}{c|}{\textbf{CIFAR-10}}
& \multicolumn{3}{c|}{\textbf{OCTMNIST}}
& \multicolumn{3}{c}{\textbf{OrganAMNIST}} \\

\cline{3-14}
& & $\alpha=0.1$ & $\alpha=0.3$ & $\alpha=0.5$
  & $\alpha=0.1$ & $\alpha=0.3$ & $\alpha=0.5$
  & $\alpha=0.1$ & $\alpha=0.3$ & $\alpha=0.5$
  & $\alpha=0.1$ & $\alpha=0.3$ & $\alpha=0.5$ \\

\hline

FedAvg & No Gating
& $-0.0636_{\pm 0.0295}$ & $-0.2545_{\pm 0.0629}$ & $-0.1921_{\pm 0.0353}$
& $-0.0416_{\pm 0.0211}$ & $-0.3540_{\pm 0.0431}$ & $-0.2840_{\pm 0.0308}$
& $-0.0285_{\pm 0.0254}$ & $-0.3003_{\pm 0.1252}$ & $-0.2001_{\pm 0.0448}$
& $-0.0495_{\pm 0.0187}$ & $-0.3193_{\pm 0.0669}$ & $-0.2408_{\pm 0.0404}$ \\

FedAvg & ENTROPY
& $-0.0201_{\pm 0.0139}$ & $-0.0133_{\pm 0.0103}$ & $\underline{-0.0006_{\pm 0.0043}}$
& $-0.0096_{\pm 0.0091}$ & $-0.0078_{\pm 0.0073}$ & $-0.0056_{\pm 0.0045}$
& $-0.0064_{\pm 0.0146}$ & $-0.0057_{\pm 0.0135}$ & $-0.0064_{\pm 0.0146}$
& $-0.0140_{\pm 0.0091}$ & $-0.0075_{\pm 0.0071}$ & $\underline{-0.0044_{\pm 0.0044}}$ \\

FedAvg & MARGIN
& $-0.0195_{\pm 0.0113}$ & $-0.0109_{\pm 0.0074}$ & $-0.0028_{\pm 0.0086}$
& $-0.0060_{\pm 0.0131}$ & $-0.0073_{\pm 0.0068}$ & $-0.0048_{\pm 0.0065}$
& $-0.0068_{\pm 0.0129}$ & $\underline{-0.0038_{\pm 0.0074}}$ & $\underline{-0.0048_{\pm 0.0089}}$
& $-0.0147_{\pm 0.0097}$ & $-0.0089_{\pm 0.0082}$ & $-0.0034_{\pm 0.0054}$ \\

FedAvg & LSE
& $-0.0108_{\pm 0.0083}$ & $\underline{-0.0070_{\pm 0.0053}}$ & $-0.0013_{\pm 0.0047}$
& $-0.0083_{\pm 0.0107}$ & $-0.0043_{\pm 0.0061}$ & $-0.0052_{\pm 0.0060}$
& $\underline{-0.0055_{\pm 0.0111}}$ & $-0.0040_{\pm 0.0109}$ & $-0.0055_{\pm 0.0111}$
& $-0.0110_{\pm 0.0094}$ & $\underline{-0.0068_{\pm 0.0069}}$ & $-0.0048_{\pm 0.0053}$ \\

FedAvg & FEAT
& \underline{$-0.0089_{\pm 0.0062}$} & $-0.0074_{\pm 0.0070}$ & $-0.0041_{\pm 0.0073}$
& $\mathbf{-0.0014}_{\pm 0.0070}$ & \underline{$-0.0028_{\pm 0.0041}$} & $\mathbf{-0.0011}_{\pm 0.0045}$
& $-0.0088_{\pm 0.0131}$ & $-0.0050_{\pm 0.0125}$ & $-0.0088_{\pm 0.0131}$
& \underline{$-0.0099_{\pm 0.0083}$} & $-0.0101_{\pm 0.0087}$ & $-0.0093_{\pm 0.0057}$ \\

\rowcolor{gray!12}
FedAvg & \textbf{\ours}
& $\mathbf{-0.0065}_{\pm 0.0066}$ & $\mathbf{-0.0035}_{\pm 0.0050}$ & $\mathbf{0.0031}_{\pm 0.0035}$
& \underline{$-0.0033_{\pm 0.0098}$} & $\mathbf{-0.0017}_{\pm 0.0051}$ & \underline{$-0.0016_{\pm 0.0048}$}
& $\mathbf{-0.0040}_{\pm 0.0084}$ & $\mathbf{-0.0013}_{\pm 0.0089}$ & $\mathbf{-0.0024}_{\pm 0.0096}$
& $\mathbf{-0.0078}_{\pm 0.0060}$ & $\mathbf{-0.0021}_{\pm 0.0064}$ & $\mathbf{-0.0024}_{\pm 0.0048}$ \\

\hline
\end{tabular}
}
\end{table}

\textbf{Entropy-based energy.}
We first consider predictive entropy as a measure of model uncertainty. Given logits $z \in \mathbb{R}^C$ and the corresponding probability vector $p = \sigma(z)$, the entropy energy is defined as
\begin{equation}
E_{\text{entropy}}(x) = - \sum_{c=1}^{C} p_c \log (p_c + \epsilon),
\end{equation}
where $\epsilon$ is a small constant for numerical stability. Low entropy corresponds to confident predictions and thus low energy, while high entropy indicates uncertainty and leads to higher energy.

\textbf{Margin-based energy.}
Margin energy measures the confidence gap between the top two predicted classes. Let $z_{(1)}$ and $z_{(2)}$ denote the largest and second-largest logits, respectively. The margin-based energy is defined as
\begin{equation}
E_{\text{margin}}(x) = - \left( z_{(1)} - z_{(2)} \right).
\end{equation}
A larger margin indicates more confident predictions and thus lower energy, while a small margin reflects ambiguity and results in higher energy.

\textbf{LogSumExp (LSE) energy.}
We also consider the classical energy-based score derived from the LogSumExp operator:
\begin{equation}
E_{\text{LSE}}(x) = - \log \sum_{c=1}^{C} \exp(z_c).
\end{equation}
This formulation is widely used in energy-based models and out-of-distribution detection, where lower values correspond to higher confidence.

\textbf{Feature-distance (FEAT) energy.}
Finally, we consider a feature-space discrepancy between the proxy and private models. Let $f_{\text{proxy}}(x)$ and $f_{\text{private}}(x)$ denote the intermediate feature representations extracted from the two models. The feature-distance energy is defined as
\begin{equation}
E_{\text{feat}}(x) = \|h_{\text{proxy}}(x) - h_{\text{private}}(x)\|_2,
\end{equation}
where $h_{\text{private}}(x)$ denotes the projected private feature representation, obtained via a linear projection layer to match the proxy feature dimension. This formulation directly measures representation-level disagreement rather than output-level uncertainty. 

\section{Sensitivity to $\lambda_{\text{kd}}$} \label{sec:sensitive_2_lambda}

\begin{figure}[H]
    \centering
    \includegraphics[width=\linewidth]{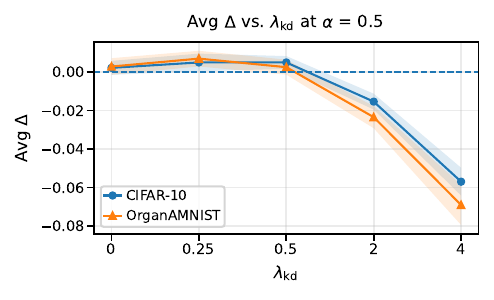}
    \caption{Avg $\Delta$ across $\lambda_{\text{kd}}$ values on CIFAR-10 and OrganAMNIST at $\alpha = 0.5$. In practice, $\lambda_{\text{kd}} \in [0.25, 0.5]$ works better under this setting.}
    \label{fig:figure4}
\end{figure}

%%%%%%%%%%%%%%%%%%%%%%%%%%%%%%%%%%%%%%%%%%%%%%%%%%%%%%%%%%%%

\end{document}